\documentclass[runningheads]{llncs}
\usepackage[T1]{fontenc}
\usepackage[utf8x]{inputenx}
\usepackage[british]{babel}

\makeatletter
\RequirePackage[bookmarks,unicode,colorlinks=true]{hyperref}%
\def\@citecolor{blue}%
\def\@urlcolor{blue}%
\def\@linkcolor{blue}%

\def\orcidID#1{\smash{\href{http://orcid.org/#1}{\protect\raisebox{-1.25pt}{\protect\includegraphics{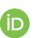}}}}}
\makeatother

\usepackage{graphicx}
\usepackage{tikz}
\usepackage{pgfplots} 
\pgfplotsset{compat=1.15}
\usetikzlibrary{%
    arrows, 
    automata, 
    fit, 
    positioning, 
    shapes, 
    calc, 
}
\usetikzlibrary{decorations,arrows,shapes,automata,calc}
\usepackage{amsmath}
\usepackage{graphicx}
\usepackage{xspace}
\usepackage{microtype}
\usepackage[vlined,linesnumbered]{algorithm2e}
\SetAlCapSkip{2ex plus 0.5ex minus 1ex}
\SetEndCharOfAlgoLine{}
\SetArgSty{textrm}
\SetKwInOut{Input}{Input}\SetKwInOut{Output}{Output}

\SetCommentSty{commentstyle}
\SetKw{Break}{break}
\SetKwProg{Type}{type}{}{end}
\SetKwProg{Function}{function}{}{end}
\SetKwFor{ForEach}{foreach}{do}{}
\SetKwFor{WhileTrue}{repeat}{}{end}
\SetKw{Continue}{continue}
\SetKwRepeat{DoWhile}{repeat}{while}
\SetKwRepeat{DoUntil}{repeat}{until}

\usepackage{hyperref} 
\usepackage{array}
\usepackage[inline]{enumitem}
\usepackage{multirow}
\usepackage{stmaryrd}
\usepackage{tabularx}
\usepackage{amssymb}
\usepackage{subfigure}
\usepackage{nicefrac}
\usepackage{xcolor}
\usepackage{colonequals}
\usepackage{adjustbox}
\usepackage{mathtools}
\usepackage[misc]{ifsym}
\usepackage{thm-restate}
\usepackage{wrapfig}
\usepackage{cleveref} 
\usepackage{tqconstraints}
\usepackage{bbding}

\newcommand{\ifempty}[3]{\ifthenelse{\equal{#1}{}}{#2}{#3}}

\definecolor{heuristiccolor}{RGB}{97,33,88}
\definecolor{commentcolor}{RGB}{60,114,26}
\definecolor{colorProbability}{RGB}{200,135,0}
\definecolor{colorAction}{RGB}{227,0,102}
\definecolor{coloraltobs}{RGB}{0,152,161}
\definecolor{colorobs}{RGB}{255,237,0}
\definecolor{colorReward}{RGB}{122,111,172}

\tikzstyle{labelnode}=[font=\footnotesize]
\tikzstyle{dist}=[circle,  inner sep=1pt, fill]
\tikzstyle{mdp}=[every label/.style={labelnode}]
\tikzset{state/.style={draw,circle,text centered,minimum size=6mm,text width=#1},state/.default=3mm}
\tikzset{ostate/.style={state=#1,rectangle, fill=colorobs!50},ostate/.default=3mm}
\tikzset{altostate/.style={state=#1,fill=coloraltobs!50},altostate/.default=3mm}
\tikzstyle{trans}=[->,semithick,labelnode]
\tikzstyle{every initial by arrow}=[inner sep=0pt] 
\tikzset{init/.style={initial #1, initial text={}, initial distance=4mm},init/.default=left}
\newcommand*{\toptionalfrac}[2][]{\ifempty{#1}{#2}{\nicefrac{#1}{#2}}}
\newcommand*{\tprob}[2][]{{\ensuremath{\color{colorProbability}\toptionalfrac[#1]{#2}}}}
\newcommand*{\tact}[1]{\ensuremath{\color{colorAction}#1}}
\newcommand*{\trew}[1]{{\ensuremath{\color{colorReward}\arraycolsep=-1pt\renewcommand*{\arraystretch}{0.75}\begin{array}{r!{\colon}l}#1\end{array}}}}
\newcommand*{\tbel}[1]{{\ensuremath{\arraycolsep=1pt\renewcommand*{\arraystretch}{0.75}\begin{array}{r!{\mapsto}l}#1\end{array}}}}

\spnewtheorem{assumption}{Assumption}{\bfseries}{\itshape}

\Crefname{figure}{Fig.}{Figs.}
\crefname{figure}{fig.}{figs.}
\Crefname{tabular}{Tab.}{Tabs.}
\crefname{tabular}{tab.}{tabs.}
\Crefname{section}{Sect.}{Sects.}
\crefname{section}{sect.}{sects.}
\Crefname{definition}{Def.}{Defs.}
\Crefname{problem}{Problem}{Problems}
\Crefname{assumption}{Assumption}{Assumptions}

\setitemize{noitemsep,topsep=0pt,parsep=0pt,partopsep=0pt,leftmargin=11.0pt}
\setenumerate{noitemsep,topsep=0pt,parsep=0pt,partopsep=0pt,leftmargin=12.5pt}

\newcommand{\tool}[1]{\textsc{#1}\xspace}
\newcommand{\model}[1]{\textsf{#1}\xspace}
\newcommand{\eg}{e.g.\ }
\newcommand{\ie}{i.e.\ }

\newcommand{\tuple}[1]{\ensuremath{\left\langle #1 \right\rangle}}

\newcommand{\set}[1]{\ensuremath{\left\{ #1 \right\}}}

\newcommand{\rr}{\ensuremath{\mathbb{R}}}

\newcommand{\rrinf}{\ensuremath{\rr^\infty}}
\newcommand{\nn}{\ensuremath{\mathbb{N}}}

\newcommand{\dist}{\ensuremath{\mu}}
\newcommand{\dists}[1]{\ensuremath{\mathit{Dist(#1)}}}
\newcommand{\supp}[1]{\ensuremath{\mathit{supp}(#1)}}
\DeclareMathOperator*{\argmin}{arg\,min}

\newcommand{\nextbelief}[3]{\ensuremath{\llbracket{#1}|{#2},{#3}\rrbracket}}

\newcommand{\mdp}{\ensuremath{M}}

\newcommand{\states}{\ensuremath{S}}

\newcommand{\actions}{\ensuremath{\mathit{Act}}}
\newcommand{\transitions}{\ensuremath{\mathbf{P}}}
\newcommand{\sinit}[1][]{\ensuremath{{s#1_{\mathit{init}}}}}
\newcommand{\mdptuple}{\ensuremath{ \tuple{\states, \actions, \transitions, \sinit} }}
\newcommand{\state}{\ensuremath{s}}
\newcommand{\action}{\ensuremath{\alpha}}
\newcommand{\act}[1]{\ensuremath{\mathit{Act}\ifthenelse{\equal{#1}{}}{}{(#1)}}}
\newcommand{\post}[3][\mdp]{\ensuremath{\mathit{post}^{#1}(#2,#3)}}

\newcommand{\goalstates}{\ensuremath{G}}

\newcommand{\goalbels}{\ensuremath{\goalstates}_{\beliefs}}
\newcommand{\threshold}{\ensuremath{\lambda}}

\newcommand{\rewards}{\ensuremath{\mathbf{R}}}

\newcommand{\pomdp}{\ensuremath{\mathcal{M}}}

\newcommand{\observations}{\ensuremath{Z}}
\newcommand{\observation}{\ensuremath{z}}
\newcommand{\obsfunction}{\ensuremath{O}}
\newcommand{\obsof}[1]{\obsfunction(#1)}
\newcommand{\pomdptuple}{\ensuremath{ \tuple{\mdp,\observations,\obsfunction}}}

\newcommand{\beliefs}{\ensuremath{\mathcal{B}}}

\newcommand{\beliefstates}{\ensuremath{B}}
\newcommand{\belieftransitions}{\ensuremath{\transitions^\beliefstates}}
\newcommand{\belief}{\ensuremath{b}}
\newcommand{\beliefstate}{\belief}
\newcommand{\binit}{\ensuremath{\beliefstate_{\mathit{init}}}}
\newcommand{\beliefmdp}[1]{\ensuremath{\mathit{bel}(#1)}}
\newcommand{\beliefrewards}{\ensuremath{\rewards^\beliefstates}}

\newcommand{\bcut}{\belief_{\text{cut}}}

\newcommand{\cutaction}{\textsf{cut}}
\newcommand{\goalaction}{\textsf{goal}}

\newcommand{\cutgoals}{\ensuremath{G_{\text{cut}}}}

\newcommand{\underapprox}[1][]{\ensuremath{\underline{\valuefunc}\ifthenelse{\equal{#1}{}}{}{(#1)}}}
\newcommand{\overapprox}[1][]{\ensuremath{\overline{\valuefunc}\ifthenelse{\equal{#1}{}}{}{(#1)}}}
\newcommand{\inftyfunction}[1][]{\ensuremath{\underapprox^{\: -\infty}\ifthenelse{\equal{#1}{}}{}{(#1)}}}
\newcommand{\zerofunction}[1][]{\ensuremath{\underapprox^{\: \mathtt{0}}\ifthenelse{\equal{#1}{}}{}{(#1)}}}
\newcommand{\stateunderapprox}[1][]{\ensuremath{U\ifthenelse{\equal{#1}{}}{}{(#1)}}}
\newcommand{\precompunderapprox}[1][]{\ensuremath{\mathfrak{U}\ifthenelse{\equal{#1}{}}{}{(#1)}}}
\newcommand{\maxprecompunderapprox}[1][]{\ensuremath{\mathfrak{U}_{\text{max}}\ifthenelse{\equal{#1}{}}{}{(#1)}}}

\newcommand{\valuefunc}[1][]{\ensuremath{V_{#1}}}
\newcommand{\optvaluefunc}[1][]{\ensuremath{V^*}}

\newcommand{\clippingoptvaluefunc}[1][]{\ensuremath{V^{\text{clip}}_{\text{opt}}}}

\newcommand{\candidatebeliefset}{\ensuremath{\mathfrak{B}}}
\newcommand{\beliefselectvar}[1]{\ensuremath{a_{#1}}}
\newcommand{\stateinfimum}[1][]{\ensuremath{\mathfrak{L}\ifthenelse{\equal{#1}{}}{}{(#1)}}}

\newcommand{\resolution}{\ensuremath{\eta}}
\newcommand{\beliefgrid}[1]{\ensuremath{\beliefs^{\#}_{#1}}}

\newcommand{\genpath}{\ensuremath{\pi}}
\newcommand{\infpath}{\ensuremath{\tilde{\pi}}}
\newcommand{\finpath}{\ensuremath{\hat{\pi}}}

\newcommand{\last}[1]{\ensuremath{\mathit{last}(#1)}}
\newcommand{\infpaths}[1]{\ensuremath{\mathit{Paths}_\mathrm{inf}^{#1}}}
\newcommand{\finpaths}[1]{\ensuremath{\mathit{Paths}_\mathrm{fin}^{#1}}}
\newcommand{\paths}[1]{\ensuremath{\mathit{Paths}^{#1}}}

\newcommand{\eventually}{\ensuremath{\lozenge}}
\newcommand{\sched}{\ensuremath{\sigma}}
\newcommand{\scheds}[1]{\ensuremath{\Sigma^{#1}}}
\newcommand{\obsscheds}[1]{\ensuremath{\Sigma^{#1}_\textnormal{obs}}}

\newcommand{\pr}[4][]{\ensuremath{\mathsf{Pr}_{#2}^{#3}(#1\ifthenelse{\equal{#1}{}}{}{\models} \eventually #4)}}
\newcommand{\prbounded}[5][]{\ensuremath{\mathsf{Pr}_{#2}^{#3}(#1\ifthenelse{\equal{#1}{}}{}{\models} \eventually^{\leq #5} #4)}}
\newcommand{\probmeasure}[1]{\ensuremath{\mu_{#1}}}
\newcommand{\rew}[1]{\ensuremath{\mathsf{rew}_{#1}}}
\newcommand{\exprew}[4][]{\ensuremath{\mathsf{ER}_{#2}^{#3}(#1\ifthenelse{\equal{#1}{}}{}{\models} \eventually #4)}}
\newcommand{\boundedexprew}[5][]{\ensuremath{\mathsf{ER}_{#2}^{#3}(#1\ifthenelse{\equal{#1}{}}{}{\models} \eventually^{\leq #5} #4)}}

\newcommand{\clippingvalue}[1][\belief{\to}\tilde{\belief}]{\ensuremath{\Delta_{#1}}}
\newcommand{\stateclippingvalue}[1][\belief{\to}\tilde{\belief}]{\ensuremath{\delta_{#1}}}

\newcommand{\clippingmdp}[1]{\ensuremath{{\mathcal{K}_{#1}}}}

\newcommand{\clippingtransitions}{\ensuremath{\transitions^\mathcal{K}}}
\newcommand{\clippingstates}{\ensuremath{\states^\mathcal{K}}}
\newcommand{\clippingrewards}{\ensuremath{\rewards^\mathcal{K}}}
\newcommand{\clippingaction}{\textsf{clip}}

\newcommand*{\schednext}[1][\observation,\action]{\sched\langle#1\rangle}

  \definecolor{RWTHblue}{RGB}{0,83,159}
    \definecolor{RWTHblack}{RGB}{0,0,0}
    \definecolor{RWTHwhite}{RGB}{255,255,255}
    \definecolor{RWTHlightblue}{RGB}{142,186,226}
    \definecolor{RWTHgrey}{RGB}{51,51,51}
    \definecolor{RWTHlightgrey}{RGB}{204,204,204}
    \definecolor{RWTHsuperlightgrey}{RGB}{247,247,247}
    \definecolor{RWTHpetrol}{RGB}{0,97,101}
    \definecolor{RWTHteal}{RGB}{0,152,161}
    \definecolor{RWTHmaygreen}{RGB}{189,205,0}
    \definecolor{RWTHgreen}{RGB}{87,171,39}
    \definecolor{RWTHyellow}{RGB}{255,237,0}
    \definecolor{RWTHorange}{RGB}{246,168,0}
    \definecolor{RWTHmagenta}{RGB}{227,0,102}
    \definecolor{RWTHred}{RGB}{204,7,30}
    \definecolor{RWTHbordeaux}{RGB}{161,16,53}
    \definecolor{RWTHviolet}{RGB}{97,33,88}
    \definecolor{RWTHpurple}{RGB}{122,111,172}
    \definecolor{orcidlogocol}{HTML}{A6CE39}

\renewcommand{\paragraph}[1]{\smallskip\noindent\emph{#1.}}

\renewcommand{\subsubsection}[1]{\medskip\noindent\textbf{#1}}

\hypersetup{
  pdftitle = {Under-Approximating Expected Total Rewards in POMDPs}
}

\begin{document}

\title{%
Under-Approximating \\ Expected Total Rewards in POMDPs
\thanks{This work is funded by the DFG RTG 2236 ``UnRAVeL''.}
}

\author{
Alexander~Bork\inst{1}$^($\Envelope$^)$\orcidID{0000-0002-7026-228X}
\and Joost-Pieter~Katoen\inst{1}\orcidID{0000-0002-6143-1926}
\and Tim~Quatmann\inst{1}\orcidID{0000-0002-2843-5511}
}

\authorrunning{A.~Bork, J.-P.~Katoen, T.~Quatmann}
\institute{
RWTH Aachen University, Aachen, Germany\\
\email{alexander.bork@cs.rwth-aachen.de}
}

\maketitle

\begin{abstract}
We consider the problem: is the optimal expected total reward to reach a goal state in a partially observable Markov decision process (POMDP) below a given threshold?
We tackle this---generally undecidable---problem by computing under-approximations on these total expected rewards.
This is done by abstracting finite unfoldings of the infinite belief MDP of the POMDP.
The key issue is to find a suitable under-approximation of the value function.
We provide two techniques: a simple (cut-off) technique that uses a good policy on the POMDP, and a more advanced technique (belief clipping) that uses minimal shifts of probabilities between beliefs.
We use mixed-integer linear programming (MILP) to find such minimal probability shifts and experimentally show that our techniques scale quite well while providing tight lower bounds on the expected total reward.
\end{abstract}

\section{Introduction}

\paragraph{The relevance of POMDPs}
Partially observable Markov decision processes (POMDPs) originated in operations research and nowadays are a pivotal model for planning in AI~\cite{DBLP:books/aw/RN2020}.
They inherit all features of classical MDPs: each state has a set of discrete probability distributions over the states and rewards are earned when taking transitions.
However, states are \emph{not} fully observable.
Intuitively, certain aspects of the states can be identified, such as a state's colour, but states themselves cannot be observed.
This partial observability reflects, for example, a robot's view of its environment while only having the limited perspective of its sensors at its disposal.
The main goal is to obtain a policy---a plan how to resolve the non-determinism in the model---for a given objective.
The key problem here is that POMDP policies must base their decisions \emph{only} on the observable aspects (\eg colours) of states.
This stands in contrast to policies for MDPs which can make decisions dependent on the entire history of \emph{full} state information.

\paragraph{Analysing POMDPs}
Typical POMDP planning problems consider either finite-horizon objectives or infinite-horizon objectives under discounting.
Finite-horizon objectives focus on reaching a certain goal state (such as \emph{``the robot has collected all items''}) within a given number of steps.
For infinite horizons, no step bound is provided and typically rewards along a run are weighted by a discounting factor that indicates how much immediate rewards are favoured over more distant ones.
Existing techniques to treat these objectives include variations of value iteration ~\cite{sondik1971,monahan1982,eagle1984,cheng1988,zhang1998,zhang2001} and policy trees~\cite{kaelbling1998}.
Point-based techniques~\cite{pineau2003,shani2013} approximate a POMDP's value function using a finite subset of beliefs which is iteratively updated.
Algorithms include \emph{PBVI} \cite{pineau2003}, \emph{Perseus} \cite{spaan2005}, \emph{SARSOP} \cite{kurniawati2008} and \emph{HSVI} \cite{smith2004}. 
Point-based methods can treat large POMDPs for both finite- and discounted infinite-horizon objectives~\cite{shani2013}. 

\paragraph{Problem statement}
In this paper we consider the problem: \emph{is the maximal expected total reward to reach a given goal state in a POMDP below a given threshold?}
We thus consider an infinite-horizon objective \emph{without} discounting---also called an \emph{indefinite-horizon} objective.
A specific instance of the considered problem is the reachability probability to eventually reach a given goal state in a POMDP.
This problem is undecidable~\cite{madani1999,madani2003} in general.
Intuitively, this is due to the fact that POMDP policies need to consider the entire (infinite) observation history to make optimal decisions.
For a POMDP, this notion is captured by an infinite, fully observable MDP, its \emph{belief MDP}.
This MDP is obtained from observation sequences inducing probabilities of being in certain states of the POMDP.

Previously proposed methods to solve the problem are \eg to use approximate value iteration~\cite{DBLP:journals/jair/Hauskrecht00}, optimisation and search techniques~\cite{DBLP:journals/aamas/AmatoBZ10,DBLP:conf/aaai/BraziunasB04}, dynamic programming~\cite{bonet1998}, Monte Carlo simulation \cite{silver2010}, game-based abstraction \cite{winterer2017}, and machine learning \cite{carr2020,carr2019,doshi2008}.
Other approaches restrict the memory size of the policies~\cite{meuleau1999}.
The synthesis of (possibly randomised) finite-memory policies is ETR-complete\footnote{A decision problem is ETR-complete if it can be reduced to a polynomial-length sentence in the Existential Theory of the Reals (for which the satisfiability problem is decidable) in polynomial time, and there is such a reduction in the reverse direction.}~\cite{junges2018}.
Techniques to obtain finite-memory policies use \eg parameter synthesis \cite{junges2018} or satisfiability checking and SMT solving~\cite{chatterjee2016b,wang2018}.

\paragraph{Our approach}
We tackle the aforementioned problem by computing under-\linebreak approximations on maximal total expected rewards.
This is done by considering finite unfoldings of the infinite belief MDP of the POMDP, and then applying abstraction.
The key issue here is to find a suitable under-approximation of the POMDP's value function.
We provide two techniques: a simple (cut-off) technique that uses a good policy on the POMDP, and a more advanced technique (belief clipping) that uses minimal shifts of probabilities between beliefs and can be applied on top of the simple approach.
We use mixed-integer linear programming (MILP) to find such minimal probability shifts.
Cut-off techniques for indefinite-horizon objectives have been used on computation trees---rather than on the belief MDP as used here---in \emph{Goal-HSVI} \cite{horak2018}.
Belief clipping amends the probabilities in a belief to be in a state of the POMDP yielding discretised values, i.e. an abstraction of the probability range $[0,1]$ is applied.
Such grid-based approximations are inspired by Lovejoy's grid-based belief MDP discretisation method \cite{lovejoy1991}. 
They have also been used in \cite{bonet2009} in the context of dynamic programming for POMDPs, and to over-approximate the value function in model checking of POMDPs~\cite{bork2020}. 
In fact, this paper on determining lower bounds for indefinite-horizon objectives can be seen as the dual counterpart of \cite{bork2020}. 
Our key challenge---compared to the approach of \cite{bork2020}---is that the value at a certain belief cannot easily be under-approximated with a convex combination of values of nearby beliefs. On the other hand, an under-approximation can benefit from a ``good'' guess of some initial POMDP policy. In the context of \cite{bork2020}, such a guessed policy is of limited use for over-approximating values in the POMDP induced by an \emph{optimal} policy.
Although our approach is applicable to all thresholds, the focus of our work is on determining under-approximations for \emph{quantitative} objectives.
Dedicated verification techniques for the qualitative setting---almost-sure reachability---are presented in~\cite{chatterjee2010,chatterjee2016,junges2021}.

\paragraph{Experimental results}
We have implemented our cut-off and belief clipping approaches on top of the probabilistic model checker \tool{Storm}~\cite{storm} and applied it to a range of various benchmarks. 
We provide a comparison with the model checking approach in~\cite{norman2017}, and determine the tightness of our under-approximations by comparing them to over-approximations obtained using the algorithm from~\cite{bork2020}.
Our main findings from the experimental validation are:
\begin{itemize}
    \item Cut-offs often generate tight bounds while being computationally inexpensive.
    \item The clipping approach may further improve the accuracy of the approximation.
    \item Our implementation can deal with POMDPs with tens of thousands of states.
    \item Mostly, the obtained under-approximations are less than 10\% off.
\end{itemize}
\section{Preliminaries and Problem Statement}
Let $\dists{A} \colonequals \set{\dist: A \to [0,1] \ | \ \sum_{a \in A} \dist(a) = 1 }$ denote the set of probability distributions over a finite set $A$.
The set $\supp{\dist} \colonequals \set{a \in A \ | \ \dist(a) > 0}$ is the \emph{support} of $\dist \in \dists{A}$.
Let $\rrinf \colonequals \rr \cup \set{\infty, -\infty}$.
We use Iverson bracket notation, where $[x] = 1$ if the Boolean expression $x$ is \emph{true} and $[x] = 0$ otherwise.

\subsection{Partially Observable MDPs}
\begin{definition}[MDP]
\label{def:mdp}
A \emph{Markov decision process (MDP)} is a tuple $\mdp = \mdptuple$ with a (finite or infinite) set of states $\states$, a finite set of actions $\actions$, a transition function $\transitions \colon \states \times \actions \times \states \to [0,1]$ with $\sum_{\state' \in \states} \transitions(\state, \action, \state') \in \set{0,1}$ for all $\state \in \states$ and $\action \in \actions$, and an initial state $\sinit$.
\end{definition}
We fix an MDP $\mdp \colonequals \mdptuple$. For $\state \in \states$ and $\action \in \actions$, let $\post{\state}{\action} \colonequals \{\state' \in \states \mid \transitions(\state,\action,\state')>0 \}$ denote the set of $\action$-successors of $\state$ in $\mdp$.
The set of \emph{enabled actions} in $\state \in \states$ is given by $\actions(\state) \colonequals \{\action \in \actions \mid \post{\state}{\action} \neq \emptyset\}$.
\begin{definition}[POMDP]
\label{def:pomdp}
A \emph{partially observable MDP (POMDP)} is a tuple $\pomdp = \pomdptuple$, where $\mdp$ is the underlying MDP with $|\states| \in \nn$, \ie $\states$ is finite, $\observations$ is a finite set of observations, and $\obsfunction \colon \states \to \observations$ is an observation function such that $\obsfunction(\state) = \obsfunction(\state') \implies \actions(\state) = \actions(\state')$ for all $\state,\state' \in \states$.
\end{definition}
We fix a POMDP $\pomdp \colonequals \pomdptuple$ with underlying MDP $\mdp$. 
We lift the notion of enabled actions to observations $\observation \in \observations$ by setting $\actions(\observation) \colonequals \actions(\state)$ for some $\state \in \states$ with $\obsof{\state} = \observation$ which is valid since states with the same observations are required to have the same enabled actions.
The notions defined for MDPs below also straightforwardly apply to POMDPs.
\begin{remark}
More general observation functions of the form $\obsfunction: \states \times \actions \to \dists{\observations}$ can be encoded in this formalism by using a polynomially larger state space \cite{chatterjee2016}.
\end{remark}
An \emph{infinite path} through an MDP (and a POMDP) is a sequence $\infpath = \state_0 \action_1 \state_1 \action_2 \hdots$ such that $\action_{i+1} \in \act{\state_{i}}$ and $\state_{i+1} \in \post{\state_{i}}{\action_{i+1}}$ for all $i \in \nn$.
A \emph{finite path} is a finite prefix $\finpath = \state_0 \action_1   \hdots \action_n \state_n$ of an infinite path $\infpath$.
For finite $\finpath$ let $\last{\finpath} \colonequals s_n$ and $|\finpath| \colonequals n$.
For infinite $\infpath$ set $|\infpath| \colonequals \infty$ and let $\infpath[i]$ denote the finite prefix of length $i \in \nn$.
We denote the set of finite and infinite paths in $\mdp$ by $\finpaths{\mdp}$ and $\infpaths{\mdp}$, respectively. Let $\paths{\mdp} \colonequals \finpaths{\mdp} \cup \infpaths{\mdp}$. 
Paths are lifted to the observation level by \emph{observation traces}. The observation trace of a (finite or infinite) path $\genpath = \state_0 \action_1 \state_1 \action_2 \hdots \in \paths{\mdp}$ is $\obsof{\genpath} \colonequals \obsof{\state_0} \action_1 \obsof{\state_1} \action_2 \hdots$.
Two paths $\genpath,\genpath' \in \paths{\mdp}$ are \emph{observation-equivalent} if $\obsof{\genpath} = \obsof{\genpath'}$.

\emph{Policies} resolve the non-determinism present in MDPs (and POMDPs).
Given a finite path $\finpath$, a policy determines the action to take at $\last{\finpath}$.
\begin{definition}[Policy]
\label{def:policy}
A \emph{policy} for $\mdp$ is a function $\sched : \finpaths{\mdp} \to \dists{\actions}$ such that for each path $\finpath \in \finpaths{\mdp}$, $\supp{\sched(\finpath)} \subseteq \act{\last{\finpath}}$. 
\end{definition}
A policy $\sched$ is \emph{deterministic} if $| \supp{\sched(\finpath)} | = 1$ for all $\finpath \in \finpaths{\mdp}$. Otherwise it is \emph{randomised}.
$\sched$ is \emph{memoryless} if for all $\finpath, \finpath' \in \finpaths{\mdp}$ we have $ \last{\finpath} = \last{\finpath'} \implies \sched(\finpath) = \sched(\finpath')$. $\sched$ is \emph{observation-based} if for all $\finpath,\finpath' \in \finpaths{\mdp}$ it holds that $\obsof{\finpath} = \obsof{\finpath'} \implies \sched(\finpath) = \sched(\finpath')$.
We denote the set of policies for $\mdp$ by $\scheds{\mdp}$ and the set of  observation-based policies for $\pomdp$ by $\obsscheds{\pomdp}$. 
A \emph{finite-memory} policy (fm-policy) can be represented by a finite automaton where the current memory state and the state of the MDP determine the actions to take~\cite{baier2008}.

The \emph{probability measure} $\probmeasure{\mdp}^{\sched,\state}$ for paths in $\mdp$ under policy $\sched$ and initial state $\state$ is the probability measure of the Markov chain induced by $\mdp$, $\sched$, and $\state$~\cite{baier2008}.

We use \emph{reward structures} to model quantities like time, or energy consumption.
\begin{definition}[Reward Structure]
\label{def:rewardStructure}
	A \emph{reward structure} for $\mdp$ is a function $\rewards\colon  \states \times \actions \times \states \to \rr$ such that either for all $\state,\state' \in \states$, $\action \in \actions$, $\rewards(\state,\action,\state') \geq 0$ or for all $\state,\state' \in \states$, $\action \in \actions$, $\rewards(\state,\action,\state') \leq 0$ holds. In the former case, we call $\rewards$ \emph{positive}, otherwise \emph{negative}.
\end{definition}
We fix a reward structure $\rewards$ for $\mdp$.
The \emph{total reward} along a path $\genpath$ is defined as $ \rew{\mdp, \rewards}(\genpath) \colonequals \sum_{i = 1}^{|\genpath|} \rewards(\state_{i-1}, \action_{i}, \state_{i})$. 
The total reward is always well-defined---even if $\genpath$ is infinite---since all rewards are assumed to be either non-negative or non-positive.
For an infinite path $\infpath$ we define the \emph{total reward} until reaching a set of goal states $\goalstates \subseteq \states$ by
$$ \rew{\mdp, \rewards, \goalstates}(\infpath) \colonequals 
\begin{cases}
			\rew{\mdp, \rewards}(\finpath) & \text{ if }\parbox[t]{.5\textwidth}{ $\exists i \in \nn : \finpath = \infpath[i] \ \land \ \last{\finpath} \in \goalstates \ \land$ \\ $\forall j<i: \last{\infpath[j]} \notin \goalstates$,} \\
			\rew{\mdp, \rewards}(\infpath) & \text{ otherwise.}
		\end{cases}
$$ 
Intuitively, $\rew{\mdp, \rewards, \goalstates}(\infpath)$ accumulates reward along $\infpath$ until the first visit of a goal state $\state \in \goalstates$.
If no goal state is reached, reward is accumulated along the infinite path.
The \emph{expected} total reward until reaching $\goalstates$ for policy $\sched$ and state $\state$ is $$\exprew[\state]{\mdp, \rewards}{\sched}{\goalstates} \colonequals \smashoperator[r]{\int\limits_{\infpath \in \infpaths{\mdp}}}
	\rew{\mdp, \rewards, \goalstates}(\infpath) \cdot \probmeasure{\mdp}^{\sched,\state}(d\infpath).$$

Observation-based policies capture the notion that a decision procedure for a POMDP only accesses the observations and their history and not the entire state of the system. 
We are interested in reasoning about \emph{minimal} and \emph{maximal} values over \emph{all} observation-based policies.
For our explanations we focus on maximising (non-negative or non-positive) expected rewards.
Minimisation can be achieved by negating all rewards.
\begin{definition}[Maximal Expected Total Reward]
\label{def:maxExpRew}
	The \emph{maximal} expected total reward until reaching $\goalstates$ from $\state$ in POMDP $\pomdp$ is
		$$\exprew[\state]{\pomdp,\rewards}{\max}{\goalstates} \colonequals \sup_{\sched \in \obsscheds{\pomdp}} \exprew[\state]{\pomdp,\rewards}{\sched}{\goalstates}.$$
	We define $\exprew{\pomdp,\rewards}{\max}{\goalstates} \colonequals \exprew[\sinit]{\pomdp,\rewards}{\max}{\goalstates}$.
\end{definition}
The central problem of our work, the \emph{indefinite-horizon total reward problem}, asks the question whether the maximal expected total reward until reaching a goal exceeds a given threshold.
\begin{center}
\fcolorbox{black}{white}{
\parbox{0.95\textwidth}{
\vspace{-.5em}
\begin{problem}
\label{prob:pomdp}
Given a POMDP $\pomdp$, reward structure $\rewards$, set of goal states $\goalstates \subseteq \states$, and threshold $\threshold \in \rr$, decide whether $\exprew{\pomdp,\rewards}{\max}{\goalstates} \leq \threshold$.
\end{problem}
\vspace{-.5em}
}}
\end{center}

\begin{example}
\Cref{fig:pomdp} shows a POMDP $\pomdp$ with three states and two observations: $\obsof{\state_0} = \obsof{\state_1} = \tikz{\node[ostate=0mm,minimum size=2mm]{}}$ and $\obsof{\state_2} = \tikz{\node[altostate=0mm,minimum size=2mm]{}}$.
A reward of 1 is collected when transitioning from $\state_1$ to $\state_2$ via the $\tact{\beta}$-action.
All other rewards are zero.
{\makeatletter
\let\par\@@par
\par\parshape0
\everypar{}%
\begin{wrapfigure}{r}{0.27\textwidth}
\centering
\vspace{-40pt}
\begin{tikzpicture}[mdp]
\node[ostate,init] (0) {$\state_0$} ;
\node[ostate,right=1.3of 0] (1) {$\state_1$} ;
\node[altostate,below=of 0] (2) {$\state_2$} ;
\path[trans]
(0) edge node[pos=0.25,above] {\tact{\action}} node[dist] (0a) {} node[pos=0.75,above] {\tprob[1]{2}} (1)
    (0a) edge[bend left] node[below] {\tprob[1]{2}} (0)
(0) edge node[pos=0.25,left] {\tact{\beta}} node[dist] {} node[pos=0.75,left] {\tprob{1}} (2)
(1) edge node[pos=0.15,left] {\tact{\beta}} node[dist] {} node[pos=0.75,right=2pt] {\trew{\rewards & 1}} node[pos=0.65,left] {\tprob{1}} (2)
(1) edge[loop below] node[pos=0.5, right] {\tact{\action}} node[dist,yshift=2pt] {} node[pos=0.5,left] {\tprob{1}} (1)
(2) edge[loop right] node[pos=0.5, above] {\tact{\action}} node[dist,xshift=-2pt] {} node[pos=0.5,below] {\tprob{1}} (2)
;
\end{tikzpicture}
\vspace{-5pt}
\caption{POMDP~$\pomdp$}
\label{fig:pomdp}
\end{wrapfigure}
The policy that always selects $\tact{\action}$ at $\state_0$ and $\tact{\beta}$ at $\state_1$ maximizes the expected total reward to reach $\goalstates = \{\state_2\}$ but is not observation-based.
The observation-based policy that for the first $n\in \nn$ transition steps selects $\tact{\action}$ and then selects $\tact{\beta}$ afterwards yields an expected total reward of $1-(\nicefrac{1}{2})^n$. With $n \to \infty$ we obtain $\exprew{\pomdp,\rewards}{\max}{\{\state_2\}} = 1$.
\par}%
\end{example}

\noindent 
As computing maximal expected rewards exactly in POMDPs is undecidable \cite{madani2003}, we aim at under-approximating the actual value $\exprew{\pomdp,\rewards}{\max}{\goalstates}$. 
This allows us to answer our problem negatively if the computed lower bound exceeds $\threshold$.

\begin{remark}
Expected rewards can be used to describe \emph{reachability probabilities} by assigning reward 1 to all transitions entering $\goalstates$ and assigning reward 0 to all other transitions.
Our approach can thus be used to obtain lower bounds on reachability probabilities in POMDPs.
This also holds for almost-sure reachability (i.e. \emph{``is the reachability probabilty one?''}), though dedicated methods like those presented in~\cite{chatterjee2010,chatterjee2016,junges2021} are better suited for that setting.
\end{remark}

\subsection{Beliefs}
\label{sec:belMDP}
The semantics of a POMDP $\pomdp$ are captured by its (fully observable) \emph{belief MDP}.
The infinite state space of this MDP consists of \emph{beliefs} \cite{astrom1965,smallwood1973}. 
A belief is a distribution over the states of the POMDP where each component describes the likelihood to be in a POMDP state given a history of observations. 
We denote the set of all beliefs for $\pomdp$ by $\beliefs_{\pomdp} \colonequals \{\belief \in \dists{\states} \mid \forall \state,\state' \in \supp{\belief} : \obsof{\state} = \obsof{\state'}\}$ and write $\obsof{\belief} \in \observations$ for the unique observation $\obsof{\state}$ of all $\state \in \supp{\belief}$.

The belief MDP of $\pomdp$ is constructed by starting in the belief corresponding to the initial state and computing successor beliefs to unfold the MDP. 
Let $\transitions(\state,\action,\observation) \colonequals \sum_{\state' \in \states} [\obsof{\state'} = \observation] \cdot \transitions(\state,\action,\state')$ be the probability to observe $\observation \in \observations$ after taking action $\action$ in POMDP state $\state$.
Then, the probability to observe $\observation$ after taking action $\action$ in belief $\belief$ is $\transitions(\belief,\action,\observation) \colonequals \sum_{\state \in \states} \belief(s) \cdot \transitions(\state,\action,\observation)$. We refer to $\nextbelief{\belief}{\action}{\observation} \in \beliefs_{\pomdp}$---the belief after taking $\action$ in $\belief$, conditioned on observing $\observation$---as the \emph{$\action$-$\observation$-successor of $\belief$}. If $\transitions(\belief,\action,\observation) > 0$, it is defined component-wise as $$ \nextbelief{\belief}{\action}{\observation}(\state) \colonequals \frac{[\obsof{\state} = \observation] \cdot \sum_{\state' \in \states} \belief(\state') \cdot \transitions(\state',\action,\state)}{\transitions(\belief,\action,\observation)}$$
for all $\state \in \states$. Otherwise $\nextbelief{\belief}{\action}{\observation}$ is \emph{undefined}.
\begin{definition}[Belief MDP]
\label{def:beliefMDP}
The \emph{belief MDP} of $\pomdp$ is the MDP $\beliefmdp{\pomdp} = \tuple{\beliefs_{\pomdp}, \actions, \belieftransitions, \binit}$, where $\beliefs_{\pomdp}$ is the set of all beliefs in $\pomdp$, $\actions$ is as for $\pomdp$, $\binit \colonequals \set{\sinit \mapsto 1}$ is  the initial belief, and $\belieftransitions \colon \beliefs_{\pomdp} \times \actions \times \beliefs_{\pomdp} \to [0,1]$ is the belief transition function with
	$$\belieftransitions(\belief, \action, \belief') \colonequals 
	\begin{cases}
		\transitions(\belief, \action, \observation) & \text{ if } \belief' = \nextbelief{\belief}{\action}{\observation}, \\
		0 & \text{ otherwise.}
	\end{cases}$$
\end{definition}
We lift a POMDP reward structure $\rewards$ to the belief MDP~\cite{itoh2007}.
\begin{definition}[Belief Reward Structure]
\label{def:beliefRewardStructure}
For beliefs $\belief,\belief' \in \beliefs_{\pomdp}$ and action $\action \in \actions$, the \emph{belief reward structure $\beliefrewards$} based on $\rewards$ associated with $\beliefmdp{\pomdp}$ is given by
$$\beliefrewards(\belief,\action,\belief') \colonequals \frac{\sum_{\state \in \states} \belief(\state) \cdot \sum_{\state' \in \states} [\obsof{s'} = \obsof{\belief'}] \cdot \rewards(\state,\action,\state')\cdot \transitions(\state,\action,\state')}{\transitions(\belief,\action,\obsof{\belief'})}.$$
\end{definition}

Given a set of goal states $\goalstates \subseteq \states$, we assume---for simplicity---that there is a set of observations $\observations' \subseteq \observations$ such that $\state \in \goalstates$ iff $\obsof{\state} \in \observations'$.
This assumption can always be ensured by transforming the POMDP $\pomdp$.
See \Cref{app:goal} for details.
The set of \emph{goal beliefs} for $\goalstates$ is given by $\goalbels \colonequals \{\belief \in \beliefs_{\pomdp} \mid \supp{\belief} \subseteq \goalstates\}$.

We now lift the computation of expected rewards to the belief level.
Based on the well-known Bellman equations~\cite{Bellman1957}, the belief MDP induces a function that maps every belief to the expected total reward accumulated from that belief.
\begin{definition}[POMDP Value Function]
\label{def:valueFunction}
For $\belief \in \beliefs_{\pomdp}$, the \emph{$n$-step value function $\valuefunc[n]: \beliefs_{\pomdp} \to \rr$} of $\pomdp$ is defined recursively as $\valuefunc[0](\belief) \colonequals 0$ and
$$
\allowdisplaybreaks
\valuefunc[n](\belief) \colonequals 
[\belief \notin \goalbels] \cdot \max_{\action \in \actions} \smashoperator[r]{\sum_{\belief' \in \post[\beliefmdp{\pomdp}]{\belief}{\action}}} \ \belieftransitions(\belief, \action, \belief') \cdot \left(\beliefrewards(\belief, \action, \belief') + \valuefunc[n-1](\belief')\right).
$$
The \emph{(optimal) value function $\optvaluefunc: \beliefs_{\pomdp} \to \rrinf$} is given by $ \optvaluefunc(\belief) \colonequals \lim_{n \to \infty} \valuefunc[n](b)$.
\end{definition}
The $n$-step value function is piecewise linear and convex \cite{smallwood1973}. Thus, the optimal value function can be approximated arbitrarily close by a piecewise linear convex function \cite{sondik1978}.
The value function yields expected total rewards in $\pomdp$ and $\beliefmdp{\pomdp}$:
$$\exprew[\state]{\pomdp,\rewards}{\max}{\goalstates} ~=~ \exprew[\set{\state \mapsto 1}]{\beliefmdp{\pomdp},\beliefrewards}{\max}{\goalbels} ~=~ \optvaluefunc(\set{\state \mapsto 1}).$$
\begin{example}
\Cref{fig:belmdp} shows a fragment of the belief MDP of the POMDP from \Cref{fig:pomdp}.
Observe $\exprew{\beliefmdp{\pomdp},\beliefrewards}{\max}{\set{\state_2 \mapsto 1}} = 1$.
\end{example}
\begin{figure}[t]
\centering
\begin{tikzpicture}[mdp]
\node[ostate=12mm,init] (0) {\tbel{\state_0 & 1\\ \state_1 & 0}} ;
\node[ostate=12mm,right=of 0] (1) {\tbel{\state_0 & \nicefrac{1}{2}\\ \state_1 & \nicefrac{1}{2}}} ;
\node[ostate=12mm,right=of 1] (2) {\tbel{\state_0 & \nicefrac{1}{4}\\ \state_1 & \nicefrac{3}{4}}} ;
\node[ostate=12mm,right=of 2] (3) {\tbel{\state_0 & \nicefrac{1}{8}\\ \state_1 & \nicefrac{7}{8}}} ;
\node[right=of 3] (dots) {~~$\cdots$~~} ;

\node[altostate=10mm, ellipse,below=of 2] (bot) {\tbel{\state_2 & 1}} ;
\path[trans]
(0) edge node[pos=0.25,above] {\tact{\action}} node[dist] {} node[pos=0.75,above] {\tprob{1}} (1)
(1) edge node[pos=0.25,above] {\tact{\action}} node[dist] {} node[pos=0.75,above] {\tprob{1}} (2)
(2) edge node[pos=0.25,above] {\tact{\action}} node[dist] {} node[pos=0.75,above] {\tprob{1}} (3)
(3) edge node[pos=0.25,above] {\tact{\action}} node[dist] {} node[pos=0.75,above] {\tprob{1}} (dots)

(0) edge[bend right] node[pos=0.15,left=2pt] {\tact{\beta}} node[dist,pos=0.3] {} node[pos=0.6,above] {\tprob{1}} node[pos=0.7,below] {\trew{\beliefrewards & 0}} (bot)
(1) edge[bend right=15] node[pos=0.2,left=2pt] {\tact{\beta}} node[dist,pos=0.3] {} node[pos=0.6,above] {\tprob{1}} node[pos=0.85,left=4pt] {\trew{\beliefrewards & \nicefrac{1}{2}}} (bot)
(2) edge[] node[pos=0.25,left] {\tact{\beta}} node[dist,pos=0.5] {} node[pos=0.75,left] {\tprob{1}} node[pos=0.75,right] {\trew{\beliefrewards & \nicefrac{3}{4}}} (bot)
(3) edge[bend left=15] node[pos=0.2,right=2pt] {\tact{\beta}} node[dist,pos=0.3] {} node[pos=0.6,above] {\tprob{1}} node[pos=0.85,right=4pt] {\trew{\beliefrewards & \nicefrac{7}{8}}} (bot)
(bot) edge[in=-20,out=0,loop] node[pos=0.4, above] {\tact{\action}} node[dist] {} node[pos=0.6,below] {\tprob{1}} (bot)
;
\end{tikzpicture}
\caption{Belief MDP $\beliefmdp{\pomdp}$ of POMDP $\pomdp$ from \Cref{fig:pomdp}}
\label{fig:belmdp}
\end{figure}
We reformulate our problem statement to focus on the belief MDP.
\begin{center}
\fcolorbox{black}{white}{
\parbox{0.95\textwidth}{
\vspace{-.5em}
\begin{problem}[equivalent to Problem \ref{prob:pomdp}]
\label{prob:belmdp}
For a POMDP $\pomdp$, reward structure $\rewards$, goal states $\goalstates \subseteq \states$, and threshold $\threshold \in \rr$, decide whether $\optvaluefunc(\set{\sinit \mapsto 1}) \leq \threshold$.
\end{problem}
\vspace{-.5em}
}}
\end{center}
As the belief MDP is fully observable, standard results for MDPs apply.
However, an exhaustive analysis of $\beliefmdp{\pomdp}$ is intractable since the belief MDP is---in general---infinitely large%
\footnote{The set of all beliefs---\ie the state space of $\beliefmdp{\pomdp}$---is uncountable. The reachable fragment is countable, though, since each belief has at most $|\observations|$ many successors.}.
\section{Finite Exploration Under-Approximation}
Instead of approximating values directly on the POMDP, we consider approximations of the corresponding belief MDP. The basic idea is to construct a finite abstraction of the belief MDP by unfolding parts of it and approximate values at beliefs where we decide not to explore. In the resulting finite MDP, under-approximative expected reward values can be computed by standard model checking techniques.
We present two approaches for abstraction: \emph{belief cut-offs} and \emph{belief clipping}.
We incorporate those techniques into an algorithmic framework that yields arbitrarily tight under-approximations.

Formal proofs of our claims are given in \Cref{app:proof}.

\subsection{Belief Cut-Offs}
\label{sec:belCutoff}
The general idea of \emph{belief cut-offs} is to stop exploring the belief MDP at certain beliefs---the \emph{cut-off beliefs}---and assume that a goal state is immediately reached while sub-optimal reward is collected. Similar techniques have been discussed in the context of fully observable MDPs and other model types \cite{brazdil2014,jansen2016,volk2017,ashok2018}. Our work adapts the idea of cut-offs for POMDP \emph{over-approximations} described in \cite{bork2020} to under-approximations.
The main idea of belief cut-offs shares similarities with the \emph{SARSOP} \cite{kurniawati2008} and \emph{Goal-HSVI} \cite{horak2018} approaches. While they apply cut-offs on the level of the computation tree, our approach directly manipulates the belief MDP to yield a finite model.

Let $\underapprox: \beliefs_{\pomdp} \to \rrinf$ with $\underapprox(\belief) \leq \optvaluefunc(\belief)$ for all $\belief \in \beliefs_{\pomdp}$. We call $\underapprox$ an \emph{under-approximative value function} and $\underapprox(\belief)$ the \emph{cut-off value} of $\belief$. In each of the cut-off beliefs $\belief$, instead of adding the regular transitions to its successors, we add a transition with probability 1 to a dedicated goal state $\bcut$. In the modified reward structure $\rewards'$, this \emph{cut-off transition} is assigned a reward%
\footnote{We slightly deviate from \Cref{def:rewardStructure} by allowing transition rewards to be $-\infty$ or $+\infty$. Alternatively, we could introduce new sink states with a non-zero self-loop reward.}
 of $\underapprox(\belief)$, causing the value for a cut-off belief $\belief$ in the modified MDP to coincide with $\underapprox(\belief)$.
Hence, the exact value of the cut-off belief---and thus the value of all other explored beliefs---is under-approximated.

\begin{example}
\Cref{fig:cut} shows the resulting \emph{finite} MDP obtained when considering the belief MDP from \Cref{fig:belmdp} with single cut-off belief $\belief = \set{\state_0 \mapsto \nicefrac{1}{4},~\state_1 \mapsto \nicefrac{3}{4}}$.
\end{example}
\begin{figure}[t]
\centering
\begin{tikzpicture}[mdp]
\node[ostate=12mm,init] (0) {\tbel{\state_0 & 1\\ \state_1 & 0}} ;
\node[ostate=12mm,right=of 0] (1) {\tbel{\state_0 & \nicefrac{1}{2}\\ \state_1 & \nicefrac{1}{2}}} ;
\node[ostate=12mm,right=of 1] (2) {\tbel{\state_0 & \nicefrac{1}{4}\\ \state_1 & \nicefrac{3}{4}}} ;

\node[altostate=10mm, ellipse,right=2 of 2] (cut) {$\bcut$} ;
\node[altostate=10mm, ellipse,below=of 2] (bot) {\tbel{\state_2 & 1}} ;

\path[trans]
(0) edge node[pos=0.25,above] {\tact{\action}} node[dist] {} node[pos=0.75,above] {\tprob{1}} (1)
(1) edge node[pos=0.25,above] {\tact{\action}} node[dist] {} node[pos=0.75,above] {\tprob{1}} (2)
(2) edge node[pos=0.25,above] {\tact{\cutaction}} node[dist] {} node[pos=0.75,above] {\tprob{1}} node[pos=0.75, below] {\trew{\rewards' & \underapprox(\belief)}} (cut)

(0) edge[bend right] node[pos=0.15,left=2pt] {\tact{\beta}} node[dist,pos=0.3] {} node[pos=0.6,above] {\tprob{1}} node[pos=0.7,below] {\trew{\rewards' & 0}} (bot)
(1) edge[bend right=15] node[pos=0.2,left=2pt] {\tact{\beta}} node[dist,pos=0.3] {} node[pos=0.6,above] {\tprob{1}} node[pos=0.85,left=4pt] {\trew{\rewards' & \nicefrac{1}{2}}} (bot)
(bot) edge[in=-10,out=10,loop] node[pos=0.4, above] {\tact{\action}} node[dist] {} node[pos=0.6,below] {\tprob{1}} (bot)
(cut) edge[in=-10,out=10,loop] node[pos=0.4, above] {\tact{\cutaction}} node[dist] {} node[pos=0.6,below] {\tprob{1}} (cut)
;
\end{tikzpicture}
\caption{Applying belief cut-offs to the belief MDP from \Cref{fig:belmdp}}
\label{fig:cut}
\end{figure}
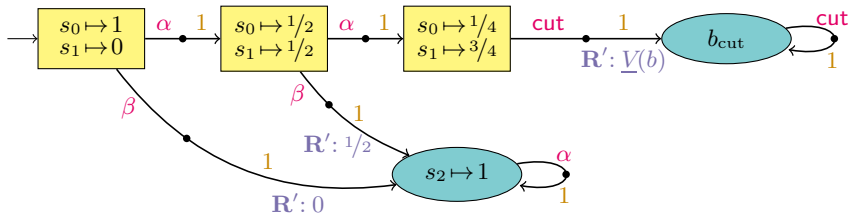

\paragraph{Computing cut-off values} The question of finding a suitable under-approximative value function $\underapprox$ is central to the cut-off approach. 
For an effective approximation, such a function should be easy to compute while still providing values close to the optimum.
If we assume a positive reward structure, the constant value $0$ is always a valid under-approximation.
A more sophisticated approach is to compute suboptimal expected reward values for the states of the POMDP using \emph{some} arbitrary, fixed observation-based policy $\sched \in \obsscheds{\pomdp}$.
Let $\stateunderapprox^\sched: \states \to \rrinf$ such that for all $\state \in \states$, $\stateunderapprox^\sched(\state) = \exprew[\state]{\pomdp,\rewards}{\sched}{\goalstates}$. 
Then, we define the function $\precompunderapprox^\sched: \beliefs_{\pomdp} \to \rrinf$ as $\precompunderapprox^\sched(\belief) \colonequals \sum_{\state \in \supp{\belief}} \belief(\state) \cdot \stateunderapprox^\sched(\state)$.
\begin{restatable}{lemma}{lemPrecompValueFunc}
\label{lem:precompValueFunc}
$\precompunderapprox^\sched$ is an under-approximative value function, \ie for all $\belief \in \beliefs_{\pomdp}$:
$$\precompunderapprox^\sched(\belief) \colonequals \sum_{\state \in \supp{\belief}} \belief(\state) \cdot \stateunderapprox^\sched(\state) \leq \optvaluefunc(\belief).$$
\end{restatable}
\noindent Thus, finding a suitable under-approximative value function reduces to finding ``good'' policies for $\pomdp$, \eg by using randomly guessed fm-policies, machine learning methods~\cite{carr2020}, or a transformation to a parametric model~\cite{junges2018}.

\subsection{Belief Clipping}
\label{sec:belClipping}
The cut-off approach provides a universal way to construct an MDP which under-approximates the expected total reward value for a given POMDP. The quality of the approximation, however, is highly dependent on the under-approximative value function used. Furthermore, regions where the belief MDP slowly converges towards a belief may pose problems in practice.

As a potential remedy for these problems, we propose a different concept called \emph{belief clipping}.
Intuitively, the procedure shifts some of the probability mass of a belief $\belief$ in order to transform $\belief$ to another belief $\tilde{\belief}$.
We then connect $\belief$ to $\tilde{\belief}$ in a way that the accuracy of our approximation of the value $\optvaluefunc(\belief)$ depends only on the approximation of $\optvaluefunc(\tilde{\belief})$ and the so-called \emph{clipping value}---some notion of distance between $\belief$ and $\tilde{\belief}$ that we discuss below.
We can thus focus on exploring the successors of $\tilde{\belief}$ to obtain good approximations for both beliefs $\belief$ and $\tilde{\belief}$.
\begin{definition}[Belief Clip]\label{def:beliefclip}
For $\belief \in \beliefs_\pomdp$, we call $\dist \colon \supp{\belief} \to [0,1]$ a \emph{belief clip} if  $\forall \state \in \supp{\belief}\colon \dist(\state) \le \belief(\state)$ and $\sum(\dist) \colonequals \sum_{\state \in \supp{\belief}} \dist(\state) < 1$.
The belief  $(\belief \ominus \dist) \in \beliefs_\mdp $ induced by $\dist$ is defined by
\[
\forall \state \in \supp{\belief} \colon ~ (\belief \ominus \dist)(\state) ~\colonequals~ \frac{\belief(\state) - \dist(\state)}{1 - \sum(\dist)}.
\]
\end{definition}
Intuitively, a belief clip $\dist$ for $\belief$ describes for each $\state \in \supp{\belief}$ the probability mass that is removed (``clipped away'') from $\belief(\state)$. The induced belief is obtained when normalising the resulting values so that they sum up to one.
\begin{example}\label{ex:beliefclip}
For belief $\belief = \set{\state_0 \mapsto \nicefrac{1}{4}, \state_1 \mapsto \nicefrac{3}{4}}$, consider the two belief clips $\dist_1 = \set{\state_0 \mapsto \nicefrac{1}{4}, \state_1 \mapsto \nicefrac{1}{4}}$ and $\dist_2 = \set{\state_0 \mapsto \nicefrac{1}{4}, \state_1 \mapsto 0}$. Both induce the same belief: $(\belief \ominus\dist_1) = (\belief \ominus\dist_2) =  \set{\state_0 \mapsto 0, \state_1 \mapsto 1}$.
\end{example}
We have $\supp{(\belief \ominus \dist)} \subseteq \supp{\belief}$, which also implies  $\obsof{(\belief \ominus \dist)}= \obsof{\belief}$.
Given some candidate belief $\tilde{\belief}$, consider the set of inducing belief clips:
$$\mathcal{C}(\belief,\tilde{\belief}) \colonequals \set{\dist \colon \supp{\belief} \to [0,1] \mid \dist \text{ is a belief clip for } \belief \text{ with } \tilde{\belief} = (\belief \ominus \dist)}.$$
Belief $\tilde{\belief}$ is called an adequate clipping candidate for $\belief$ iff $\mathcal{C}(\belief,\tilde{\belief}) \neq \emptyset$.
\begin{definition}[Clipping Value]
\label{def:clippingvalue}
For $\belief \in \beliefs_\pomdp$ and adequate clipping candidate $\tilde{\belief}$, the \emph{clipping value} is $\clippingvalue \colonequals \sum(\stateclippingvalue)$, where $\stateclippingvalue \colonequals \argmin_{\dist \in \mathcal{C}(\belief,\tilde{\belief})} \sum(\dist)$.
The values $\stateclippingvalue(\state)$ for $\state \in \supp{\belief}$ are the \emph{state clipping values}.
\end{definition}
Given a belief $\belief$ and an adequate clipping candidate $\tilde{\belief}$, we outline how the notion of belief clipping is used to obtain valid under-approximations.
We assume $\belief \neq \tilde{\belief}$, implying $0 < \clippingvalue < 1$.
Instead of exploring all successors of $\belief$ in $\beliefmdp{\pomdp}$, the approach is to add a transition from $\belief$ to $\tilde{\belief}$.
The newly added transition has probability $1-\clippingvalue$ and gets assigned a reward of 0.
The remaining probability mass (\ie $\clippingvalue$) leads to a designated goal state $\bcut$.
To guarantee that---in general---the clipping procedure yields a valid under-approximation, we need to add a corrective reward value to the transition from $\belief$ to $\bcut$.
Let $\stateinfimum: \states \to \rrinf$ which maps each POMDP state to its \emph{minimum} expected reward in the underlying, fully observable  MDP $\mdp$ of $\pomdp$%
\footnote{When rewards are negative, we might have  $\stateinfimum(\state) = -\infty$ for many $\state \in \states\setminus\goalstates$ in which case the applicability of the clipping approach is very limited.
}%
, \ie $\stateinfimum(\state) = \exprew[\state]{\mdp, \rewards}{\min}{\goalstates}$. This function soundly under-approximates the state values which can be achieved by \emph{any} observation-based policy.
It can be generated using standard MDP analysis.
Given state clipping values $\stateclippingvalue(\state)$ for $\state \in \supp{\belief}$, the reward for the transition from $\belief$ to $\bcut$ is $\sum_{\state \in \supp{\belief}} (\stateclippingvalue(\state) / \clippingvalue) \cdot \stateinfimum[\state].$
\begin{example}
For the belief MDP from \Cref{fig:belmdp}, belief $\belief = \set{\state_0 \mapsto \nicefrac{1}{4},~\state_1 \mapsto \nicefrac{3}{4}}$, and clipping candidate $\tilde{\belief} =  \set{\state_0 \mapsto 0,~\state_1 \mapsto 1}$ we get $\clippingvalue = \nicefrac{1}{4}$, as
$\stateclippingvalue = \dist_2 = \set{\state_0 \mapsto \nicefrac{1}{4},~\state_1\mapsto 0}$ with the belief clip $\dist_2$ as in \Cref{ex:beliefclip}. Furthermore, $\stateinfimum(\state_0) = 0$.
The resulting MDP following our construction above is given in \Cref{fig:clip}.
\end{example}
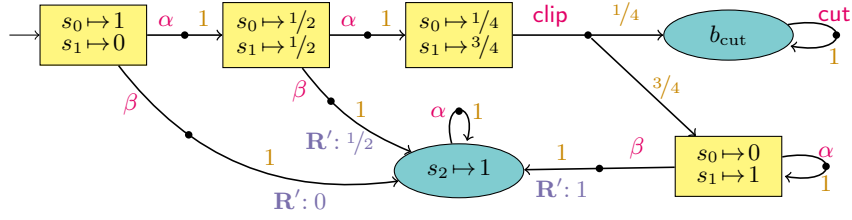
\begin{figure}[t]
\centering
\begin{tikzpicture}[mdp]
\node[ostate=12mm,init] (0) {\tbel{\state_0 & 1\\ \state_1 & 0}} ;
\node[ostate=12mm,right=of 0] (1) {\tbel{\state_0 & \nicefrac{1}{2}\\ \state_1 & \nicefrac{1}{2}}} ;
\node[ostate=12mm,right=of 1] (2) {\tbel{\state_0 & \nicefrac{1}{4}\\ \state_1 & \nicefrac{3}{4}}} ;
\node[altostate=10mm, ellipse,right=2 of 2] (cut) {$\bcut$} ;
\node[ostate=12mm,below=of cut] (3) {\tbel{\state_0 & 0\\ \state_1 & 1}} ;
\node[altostate=10mm, ellipse,below=of 2] (bot) {\tbel{\state_2 & 1}} ;
\path[trans]
(0) edge node[pos=0.25,above] {\tact{\action}} node[dist] {} node[pos=0.75,above] {\tprob{1}} (1)
(1) edge node[pos=0.25,above] {\tact{\action}} node[dist] {} node[pos=0.75,above] {\tprob{1}} (2)
(2) edge node[pos=0.25,above] {\tact{\clippingaction}} node[dist] (2a) {} node[pos=0.75,above] {\tprob[1]{4}}  (cut)
(2a) edge node[pos=0.5,right] {\tprob[3]{4}}  (3)
(0) edge[bend right] node[pos=0.15,left=2pt] {\tact{\beta}} node[dist,pos=0.3] {} node[pos=0.6,above] {\tprob{1}} node[pos=0.7,below] {\trew{\rewards' & 0}} (bot)
(1) edge[bend right=15] node[pos=0.2,left=2pt] {\tact{\beta}} node[dist,pos=0.3] {} node[pos=0.6,above] {\tprob{1}} node[pos=0.85,left=4pt] {\trew{\rewards' & \nicefrac{1}{2}}} (bot)
(3) edge[] node[pos=0.25,above] {\tact{\beta}} node[dist,pos=0.5] {} node[pos=0.75,above] {\tprob{1}} node[pos=0.75,below] {\trew{\rewards' &1}} (bot)
(3) edge[in=-10,out=10,loop] node[pos=0.5, above] {\tact{\action}} node[dist] {} node[pos=0.5,below] {\tprob{1}} (3)
(bot) edge[loop above] node[pos=0.4, left] {\tact{\action}} node[dist,yshift=-2pt] {} node[pos=0.6,right] {\tprob{1}} (bot)
(cut) edge[in=-10,out=10,loop] node[pos=0.4, above] {\tact{\cutaction}} node[dist] {} node[pos=0.6,below] {\tprob{1}} (cut)
;
\end{tikzpicture}
\caption{Applying belief clipping to the belief MDP from \Cref{fig:belmdp}}
\label{fig:clip}
\end{figure}
The following lemma shows that the construction yields an under-approximation.
\begin{restatable}{lemma}{lemClip}\label{lem:clip}
$\displaystyle(1-\clippingvalue) \cdot \optvaluefunc(\tilde{\belief}) ~+~ \clippingvalue  \cdot \!\sum_{\state \in \supp{\belief}} \frac{\stateclippingvalue(\state)}{\clippingvalue} \cdot \stateinfimum[\state] ~\le~\optvaluefunc(\belief)$.
\end{restatable}
\begin{proof}[sketch]
To gain some intuition, consider the special case, where $\clippingvalue = \stateclippingvalue(\state) = \belief(\state)$ for some $\state \in \supp{\belief}$. 
The clipping candidate $\tilde{\belief}$ can be interpreted as the conditional probability distribution arising from distribution $\belief$ given that $\state$ is \emph{not} the current state.
The value $\optvaluefunc(\belief)$ can be split into the sum of
\begin{enumerate*}[label=(\roman*)]
    \item the probability that $\state$ is \emph{not} the current state times the reward accumulated from belief $\tilde{\belief}$ and
    \item the probability that $\state$ \emph{is} the current state  times the reward accumulated from $\state$, \ie from the belief $\set{\state \mapsto 1}$.
\end{enumerate*}
However, for the two summands we must consider a policy that does not distinguish between the beliefs  $\belief$, $\tilde{\belief}$, and $\set{\state \mapsto 1}$ as well as their observation-equivalent successors.
In other words, the same sequence of actions must be executed when the same observations are made.

We consider such a policy that in addition is optimal at $\tilde{\belief}$, \ie the reward accumulated from $\tilde{\belief}$ is equal to $\optvaluefunc(\tilde{\belief})$.
For the reward accumulated from $\set{\state \mapsto 1}$, $\stateinfimum[\state]$ provides a lower bound.
Hence, $(1-\belief(\state)) \cdot \optvaluefunc(\belief) + \belief(\state) \cdot \stateinfimum[\state]$ is a lower bound for the reward accumulated from $\belief$.
A formal proof is given in \Cref{app:clip}.
\end{proof}
To find a suitable clipping candidate for a given belief $\belief$, we consider a finite \emph{candidate set} $\candidatebeliefset \subseteq \beliefs_{\pomdp}$ consisting of beliefs with observation $\obsof{\belief}$. These beliefs do not need to be reachable in the belief MDP. The set can be constructed, \eg by taking already explored beliefs or by using a fixed, discretised set of beliefs.

We are interested in minimising the clipping value $\clippingvalue[\belief\to\belief']$ over all candidate beliefs $\belief' \in \candidatebeliefset$.
A naive approach is to explicitly compute all clipping values for all candidates. We are using \emph{mixed-integer linear programming (MILP)}~\cite{schrijver1986} instead.\\
An MILP is a system of linear inequalities (\emph{constraints}) and a linear \emph{objective function} considering real-valued and integer variables.
A \emph{feasible solution} of the MILP is a variable assignment that satisfies all constraints.
An \emph{optimal solution} is a feasible solution that minimises the objective function.

\begin{definition}[Belief Clipping MILP]\label{def:clippingvalueMILP}
The belief clipping MILP for belief $\belief \in \beliefs_{\pomdp}$ and finite set of candidates $\candidatebeliefset \subseteq \set{\belief' \in \beliefs_{\pomdp} \mid \obsof{\belief'} = \obsof{\belief}}$ is given by:
\begin{center}
\begin{constraints}
\lpminimize{\clippingvalue[]}
\lpline{}{}{\sum_{\belief' \in \candidatebeliefset} \beliefselectvar{\belief'}}{= 1}{Select exactly one candidate $\belief'$}{eq:exactlyOne}
\lpline{}{\forall \belief' \in \candidatebeliefset\colon}{\beliefselectvar{\belief'}}{\in \set{0,1} }{}{eq:binary}
\lpline{}{}{\sum_{\state \in \supp{\belief}} \stateclippingvalue[\state]}{=\clippingvalue[]}{Compute clipping value for selected  $\belief'$}{eq:sum}
\lpline{}{\forall \state \in \supp{\belief}\colon}{\stateclippingvalue[\state]}{\in [0,\,\belief(s)]}{}{eq:stateclipbnds}
\lpline{I}{\forall \belief' \in \candidatebeliefset\colon}{\stateclippingvalue[\state]}{\ge \belief(\state)- (1 - \clippingvalue[]) \cdot \belief'(\state) - (1 - \beliefselectvar{\belief'})}{}{eq:linClip}
\end{constraints}
\end{center}
\end{definition}
The MILP consists of $\mathcal{O}(|\supp{\belief}| + |\candidatebeliefset|)$ variables and $\mathcal{O}(|\supp{\belief}| \cdot |\candidatebeliefset|)$ constraints.
For $\belief' \in \candidatebeliefset$, the binary variable $\beliefselectvar{\belief'}$ indicates whether $\belief'$ has been chosen as the clipping candidate.
Moreover, we have variables $\stateclippingvalue[\state]$ for $\state \in \supp{\belief}$ and a variable $\clippingvalue[]$ to represent the (state) clipping values for $\belief$ and the chosen candidate $\belief'$.
\Cref{eq:exactlyOne,eq:binary} enforce that exactly one of the $\beliefselectvar{\belief'}$ variables is one, \ie exactly one belief is chosen.
\Cref{eq:sum} forces $\clippingvalue[]$ to be the sum of all state clipping values.
$\stateclippingvalue[\state]$ variables get a value between zero and $\belief(s)$ (\Cref{eq:stateclipbnds}).
\Cref{eq:linClip} only affects $\stateclippingvalue[\state]$ if the corresponding belief is chosen. Otherwise, $\beliefselectvar{\belief'}$ is set to $0$ and the value on the right-hand side becomes negative.
If a belief $\belief'$ is chosen, the minimisation forces \Cref{eq:linClip} to hold with equality as the right-hand side is greater or equal to $0$. Assuming $\clippingvalue[]$ is set to a value below 1, we obtain a valid clipping values as
\[
\forall \state \in \supp{\belief} \colon \quad 
\stateclippingvalue[\state] = \belief(\state)- (1 - \clippingvalue[]) \cdot \belief'(\state)
\quad\iff\quad
\belief'(\state) = \frac{\belief(\state) - \stateclippingvalue[\state]}{1 - \clippingvalue[]}.
\]
A trivial solution of the MILP is always obtained by setting  $\beliefselectvar{\belief'}$ and $\clippingvalue[]$ to $1$ and $\stateclippingvalue[\state]$ to $\belief(\state)$ for all $\state$ and an arbitrary $\belief' \in \candidatebeliefset$. This corresponds to an invalid belief clip. However, as we minimise the value for $\clippingvalue[]$, we can conclude that \emph{no} belief in the candidate set is adequate for clipping if $\clippingvalue[]$ is $1$ in an optimal solution.

\begin{restatable}{theorem}{thmMILP}
\label{thm:MILP}
An optimal solution to the belief clipping MILP for belief $\belief$ and candidate set $\candidatebeliefset$ sets $\beliefselectvar{\tilde{\belief}}$ to 1 and $\clippingvalue[]$ to a value below 1 iff $\tilde{\belief} \in \candidatebeliefset$ is an adequate clipping candidate for $\belief$ with minimal clipping value.
\end{restatable}

\subsection{Algorithm}
\label{sec:algorithm}
\begin{algorithm}[!t]
	\SetKwFunction{solveClippingMILP}{solveClippingMILP}
	\SetKwFunction{ChooseBelief}{chooseBelief}
	\SetKwFunction{InsertBelief}{insertBelief}
	\SetKwFunction{exploreBelief}{exploreBelief}
	\SetKwFunction{handleClipping}{handleClipping}
	\Input{POMDP $\pomdp = \pomdptuple$ with $\mdp = \mdptuple$, reward structure $\rewards$, goal states $\goalstates \subseteq \states$, under-approx. value function~$\underapprox$, function $\stateinfimum: \states \to \rrinf$ with $\stateinfimum[\state] = \exprew[\state]{\mdp,\rewards}{\min}{\goalstates}$}
	\Output{Clipping belief MDP $\clippingmdp{\pomdp}$ and reward structure $\clippingrewards$}
	
	$\clippingstates \gets \set{\binit, \bcut}$ with $\binit = \set{\sinit \mapsto 1}$ and a new belief state $\bcut$ \label{line:init1}	\;
	$\clippingtransitions(\bcut,\cutaction,\bcut) \gets 1$,
	$\clippingrewards(\bcut,\cutaction,\bcut) \gets 0$ \label{line:bcutselfloop}\tcp*{add self-loop}
	$Q \gets \set{\binit}$\label{line:init2} \tcp*{initialize exploration set} 
	\While{$Q \neq \emptyset$\label{line:loopstart}}{
		$\belief \gets \ChooseBelief(Q)$, 	$Q \gets Q \setminus \set{\belief}$ \label{line:choice1} \tcp*{pop next belief to explore from $Q$} 
			\lIf(\tcp*[f]{add self-loop}\label{line:goal1}){$\supp{\belief} \subseteq \goalstates$}{%
				$\clippingtransitions(\belief,\goalaction,\belief) \gets 1$, 
				$\clippingrewards(\belief,\goalaction,\belief) \gets 0$\label{line:goal2}
			} \uElseIf(\tcp*[f]{expand $\belief$}){\exploreBelief{$\belief$}\label{line:expand}}{%
				\ForEach(\tcp*[f]{Using $\beliefmdp{\pomdp}$ and $\beliefrewards$ as in \Cref{def:beliefMDP,def:beliefRewardStructure}}){$\action \in \actions(\belief)$\label{line:expandstart}}{
 			        \ForEach{$\belief' \in \post[\beliefmdp{\pomdp}]{\belief}{\action}$}{%
 			            $\clippingtransitions(\belief,\action,\belief') \gets \belieftransitions(\belief,\action,\belief')$,
 					    $\clippingrewards(\belief,\action,\belief') \gets \beliefrewards(\belief,\action,\belief')$\;
 					    \lIf{$\belief' \notin \clippingstates$}{\label{line:notInStates1}%
     						$\clippingstates \gets \clippingstates  \cup \set{\belief'}$,
 	    					$Q \gets Q \cup \set{\belief'}$\label{line:addNewBel1}\label{line:expandend}	
 		    			}
 			        }	
 		        }
			} \Else(\tcp*[f]{apply cut-off and clipping to $\belief$}\label{line:cutclipstart}){
				$\clippingtransitions(\belief,\cutaction,\bcut) \gets 1$, 
				$\clippingrewards(\belief,\cutaction,\bcut) \gets \underapprox[\belief]$\label{line:cutoff1} \tcp*{add cut-off transition}
				choose a finite set 	$\candidatebeliefset \subseteq \beliefs_\pomdp$ of clipping candidates for $\belief$ \label{line:chooseCandidates}\\
				$\tilde{\belief},\clippingvalue, \stateclippingvalue \gets \solveClippingMILP(\belief,\candidatebeliefset)$ \label{line:solveClipping}\\	
				\If(\tcp*[f]{Clip $\belief$ using $\tilde{\belief}$} \label{line:clipping1}){$\tilde{\belief} \neq \belief$ and $\tilde{\belief}$ is adequate}{
					$\clippingtransitions(\belief,\clippingaction,\tilde{\belief}) \gets (1{-}\clippingvalue)$,
					$\clippingtransitions(\belief,\clippingaction,\bcut) \gets \clippingvalue$\;
					$\clippingrewards(\belief,\clippingaction,\tilde{\belief}) \gets 0$,
					$\clippingrewards(\belief,\clippingaction,\bcut) \gets\sum_{\state \in \supp{\belief}}  \frac{\stateclippingvalue(\state)}{\clippingvalue} \cdot \stateinfimum[\state]$\label{line:clipping2}\;
			    \lIf{$\tilde{\belief} \notin \clippingstates$}{%
     						$\clippingstates \gets \clippingstates \cup \{\tilde{\belief}\}$,
 	    					$Q \gets Q \cup \{\tilde{\belief}\}$\label{line:loopend}\label{line:cutclipend}\label{line:handlecandidate}	
 		    			}
				}
			}
	}
	\Return{$\clippingmdp{\pomdp} = \tuple{\clippingstates, \actions \uplus \set{\goalaction,\cutaction,\clippingaction}, \clippingtransitions, \binit}$ and $\clippingrewards$}

	\caption{Belief exploration algorithm with cut-offs and clipping}
	\label{alg:clippingMDP}
\end{algorithm}
We incorporate belief cut-offs and belief clipping into an algorithmic framework outlined in \Cref{alg:clippingMDP}.
As input, the algorithm takes an instance of \Cref{prob:pomdp,prob:belmdp}, \ie a POMDP $\pomdp$ with reward structure $\rewards$ and goal states $\goalstates$. In addition, the algorithm considers an under-approximative value function $\underapprox$ (\Cref{sec:belCutoff}) and a function $\stateinfimum$ for the computation of corrective reward values (\Cref{sec:belClipping}).

\Cref{line:init1,line:bcutselfloop} initialise the state set $\clippingstates$ of the under-approximative MDP $\clippingmdp{\pomdp}$ with the initial belief $\binit$ and the designated goal state $\bcut$ which has only one transition to itself with reward $0$.
Furthermore, we initialise the \emph{exploration set} $Q$ by adding $\binit$ (\Cref{line:init2}). 
During the computation, $Q$ is used to keep track of all beliefs we still need to process. We then execute the exploration loop (\Crefrange{line:loopstart}{line:loopend}) until $Q$ becomes empty.
In each exploration step, a belief $\belief$ is selected\footnote{For example, $Q$ can be implemented as a FIFO queue.} and removed from $Q$.
There are three cases for the currently processed belief $\belief$.

If $\supp{\belief} \subseteq \goalstates$, \ie $\belief$ is a goal belief, we add a self-loop with reward $0$ to $\belief$ and continue with the next belief (\Cref{line:goal1}).
$\belief$ is not expanded as successors of goal beliefs will not influence the result of the computation.

If $\belief$ is not a goal belief, we use a heuristic function\footnote{The decision can be made for example by considering the size of the already explored state space such that the expansion is stopped if a size threshold has been reached. More involved decision heuristics are subject to further research.} \exploreBelief to decide if $\belief$ is expanded in \Cref{line:expand}.
\Crefrange{line:expandstart}{line:expandend} outline the expansion step. 
The transitions from $\belief$ to its successor beliefs and the corresponding rewards as in the original belief MDP (see \Cref{sec:belMDP}) are added.
Furthermore, the successor beliefs that have not been encountered before are added to the set of states $\clippingstates$ and the exploration set $Q$.

If $\belief$ is \emph{not} expanded, we apply the cut-off approach \emph{and} the clipping approach to $\belief$ in \Crefrange{line:cutclipstart}{line:cutclipend}.
In \Cref{line:cutoff1} we add a cut-off transition from $\belief$ to $\bcut$ with a new action $\cutaction$.
We use the given under-approximative value function $\underapprox$ to compute the cut-off reward.
Towards the clipping approach, a set of candidate beliefs is chosen and the belief clipping MILP for $\belief$ and the candidate set is constructed as described in \Cref{def:clippingvalueMILP} (\Cref{line:chooseCandidates,line:solveClipping}).
If an adequate candidate $\tilde{\belief}$ with clipping values $\clippingvalue$ and $\stateclippingvalue(\state)$ for $\state \in \supp{\belief}$ has been found, we add the transitions from $\belief$  to $\bcut$ and to $\tilde{\belief}$ using a new action $\clippingaction$ and probabilities $\clippingvalue$ and $1-\clippingvalue$, respectively. 
Furthermore, we equip the transitions with reward values as described in \Cref{sec:belClipping} using the given function $\stateinfimum$ (\Crefrange{line:clipping1}{line:clipping2}).
If the clipping candidate $\tilde{\belief}$ has not been encountered before, we add it to the state space of the MDP and to the exploration set in \Cref{line:handlecandidate}.

The result of the algorithm is an MDP $\clippingmdp{\pomdp}$ with reward structure $\clippingrewards$.
The set of states $\clippingstates$ of $\clippingmdp{\pomdp}$ contains all encountered beliefs.
To guarantee termination of the algorithm, the decision heuristic \exploreBelief has to stop exploring further beliefs at some point.
Moreover, the handling of clipping candidates in \Cref{line:handlecandidate} should not add new beliefs to $Q$ infinitely often.
We therefore fix a finite set of candidate beliefs $\beliefgrid{} \subseteq \beliefs_\pomdp$ and make sure that the candidate sets $\candidatebeliefset$ in \Cref{line:chooseCandidates} satisfy $(\candidatebeliefset \setminus \clippingstates) \subseteq \beliefgrid{}$.
To ensure a certain progress in the exploration ``$\clippingaction$-cycles''---\ie paths of the form $\belief_1\, \clippingaction\, \dots\, \clippingaction \,\belief_n\, \clippingaction\, \belief_1$---are avoided in $\clippingmdp{\pomdp}$. This can be done, \eg by always expanding the candidate beliefs $\belief \in \beliefgrid{}$.

Expected total rewards until reaching the extended set of goal beliefs $\cutgoals \colonequals \goalbels \cup \{\bcut \}$ in $\clippingmdp{\pomdp}$ under-approximate the values in the belief MDP:
\begin{restatable}{theorem}{clippingUnderapprox}
\label{thm:clippingUnderapprox}
For all beliefs $\belief \in \clippingstates\setminus \set{\bcut}$ it holds that
$$\exprew[\belief]{\clippingmdp{\pomdp},\clippingrewards}{\max}{\cutgoals} \leq \optvaluefunc(\belief) = \exprew[\belief]{\beliefmdp{\pomdp},\beliefrewards}{\max}{\goalbels}.$$
\end{restatable}

\begin{restatable}{corollary}{corollaryClippingMDP}
\label{thm:clippingMDP}
$\exprew{\clippingmdp{\pomdp},\clippingrewards}{\max}{\cutgoals} \leq \exprew{\pomdp,\rewards}{\max}{\goalstates}$.
\end{restatable}
\section{Experimental Evaluation}
\begin{table}[t]
\caption{Results for benchmark POMDPs with maximisation objective}
\centering
\scriptsize
\begin{tabular}{|cc||r||r||r|r|r|r|r||r|}
\hline
\multicolumn{2}{|c||}{Benchmark} & \multicolumn{1}{c||}{Data} & \multicolumn{1}{c||}{\tool{Prism}} & \multicolumn{6}{c|}{\tool{Storm}}\\
\multicolumn{1}{|c}{Model} & \multicolumn{1}{c||}{$\phi$} & \multicolumn{1}{c||}{$\states$/$\actions$/$\observations$} & \multicolumn{1}{c||}{} & \multicolumn{1}{c|}{Cut-Off} & \multicolumn{4}{c||}{Cut-Off + Clipping} & \multicolumn{1}{c|}{Over-}\\
\multicolumn{1}{|c}{} & \multicolumn{1}{c||}{} & \multicolumn{1}{c||}{} & \multicolumn{1}{c||}{} &  \multicolumn{1}{c|}{Only} & \multicolumn{1}{c|}{$\resolution$=2} & \multicolumn{1}{c|}{$\resolution$=3} & \multicolumn{1}{c|}{$\resolution$=4} & \multicolumn{1}{c||}{$\resolution$=6} & \multicolumn{1}{c|}{Approx.}\\
\hline\hline
\model{Drone} & \multirow{3}{*}{$P_\mathrm{max}$} & $1226$ & TO / MO & ${\ge}\, 0.79$ & ${\ge}\, 0.79$ & \multirow{3}{*}{TO} & \multirow{3}{*}{TO} & \multirow{3}{*}{TO} & ${\le}\, 0.94$\\
\multirow{2}{*}{4-1} &  & $2954$ &  & \ensuremath{${<}\,$ 1\text{s}} & \ensuremath{1360\text{s}} &  &  &  & \\
 &  & $384$ &  & $3{\cdot} 10^{4}$ & $3{\cdot} 10^{4}$ &  &  &  & \\
\hline
\model{Drone} & \multirow{3}{*}{$P_\mathrm{max}$} & $1226$ & TO / MO & ${\ge}\, 0.86$ & ${\ge}\, 0.91$ & ${\ge}\, 0.92$ & \multirow{3}{*}{TO} & \multirow{3}{*}{TO} & ${\le}\, 0.97$\\
\multirow{2}{*}{4-2} &  & $2954$ &  & \ensuremath{${<}\,$ 1\text{s}} & \ensuremath{249\text{s}} & \ensuremath{1902\text{s}} &  &  & \\
 &  & $761$ &  & $2{\cdot} 10^{4}$ & $2{\cdot} 10^{4}$ & $2{\cdot} 10^{4}$ &  &  & \\
\hline\hline
\model{Grid-av} & \multirow{3}{*}{$P_\mathrm{max}$} & $17$ & $[\mathbf{0.21}, 1.0]$ & ${\ge}\, 0.86$ & ${\ge}\, 0.93$ & ${\ge}\, 0.93$ & ${\ge}\, 0.93$ & ${\ge}\, 0.93$ & ${\le}\, 0.98$\\
\multirow{2}{*}{4-0} &  & $59$ & \ensuremath{5.14\text{s}} & \ensuremath{${<}\,$ 1\text{s}} & \ensuremath{${<}\,$ 1\text{s}} & \ensuremath{1.77\text{s}} & \ensuremath{3.63\text{s}} & \ensuremath{13.9\text{s}} & \\
 &  & $4$ & $\eta=$6 & $238$ & $312$ & $472$ & $663$ & $1300$ & \\
\hline
\model{Grid-av} & \multirow{3}{*}{$P_\mathrm{max}$} & $17$ & $[\mathbf{0.21}, 1.0]$ & ${\ge}\, 0.82$ & ${\ge}\, 0.85$ & ${\ge}\, 0.82$ & ${\ge}\, 0.85$ & \multirow{3}{*}{TO} & ${\le}\, 0.99$\\
\multirow{2}{*}{4-0.1} &  & $59$ & \ensuremath{1.47\text{s}} & \ensuremath{${<}\,$ 1\text{s}} & \ensuremath{26.1\text{s}} & \ensuremath{198\text{s}} & \ensuremath{1913\text{s}} &  & \\
 &  & $4$ & $\eta=$3 & $238$ & $317$ & $461$ & $759$ &  & \\
\hline\hline
\model{Netw-p} & \multirow{3}{*}{$R_\mathrm{max}$} & $2{\cdot} 10^{4}$ & $[\mathbf{557}, 557]$ & ${\ge}\, 537$ & ${\ge}\, 537$ & ${\ge}\, 537$ & ${\ge}\, 537$ & ${\ge}\, 537$ & ${\le}\, 558$\\
\multirow{2}{*}{2-8-20} &  & $3{\cdot} 10^{4}$ & \ensuremath{2355\text{s}} & \ensuremath{2.3\text{s}} & \ensuremath{98.5\text{s}} & \ensuremath{320\text{s}} & \ensuremath{651\text{s}} & \ensuremath{2368\text{s}} & \\
 &  & $4909$ & $\eta=$10 & $8{\cdot} 10^{4}$ & $1{\cdot} 10^{5}$ & $1{\cdot} 10^{5}$ & $1{\cdot} 10^{5}$ & $1{\cdot} 10^{5}$ & \\
\hline
\model{Netw-p} & \multirow{3}{*}{$R_\mathrm{max}$} & $2{\cdot} 10^{5}$ & TO / MO & ${\ge}\, 769$ & ${\ge}\, 769$ & \multirow{3}{*}{TO} & \multirow{3}{*}{TO} & \multirow{3}{*}{TO} & ${\le}\, 819$\\
\multirow{2}{*}{3-8-20} &  & $3{\cdot} 10^{5}$ &  & \ensuremath{290\text{s}} & \ensuremath{6640\text{s}} &  &  &  & \\
 &  & $2{\cdot} 10^{4}$ &  & $1{\cdot} 10^{6}$ & $1{\cdot} 10^{6}$ &  &  &  & \\
\hline\hline
\model{Refuel} & \multirow{3}{*}{$P_\mathrm{max}$} & $208$ & $[\mathbf{0.67}, 0.72]$ & ${\ge}\, 0.67$ & ${\ge}\, 0.67$ & ${\ge}\, 0.67$ & ${\ge}\, 0.67$ & ${\ge}\, 0.67$ & ${\le}\, 0.69$\\
\multirow{2}{*}{06} &  & $565$ & \ensuremath{4625\text{s}} & \ensuremath{${<}\,$ 1\text{s}} & \ensuremath{5.89\text{s}} & \ensuremath{24.3\text{s}} & \ensuremath{92\text{s}} & \ensuremath{2076\text{s}} & \\
 &  & $50$ & $\eta=$3 & $4576$ & $4834$ & $5204$ & $5603$ & $6135$ & \\
\hline
\model{Refuel} & \multirow{3}{*}{$P_\mathrm{max}$} & $470$ & TO / MO & ${\ge}\, 0.45$ & ${\ge}\, 0.45$ & \multirow{3}{*}{TO} & \multirow{3}{*}{TO} & \multirow{3}{*}{TO} & ${\le}\, 0.51$\\
\multirow{2}{*}{08} &  & $1431$ &  & \ensuremath{${<}\,$ 1\text{s}} & \ensuremath{839\text{s}} &  &  &  & \\
 &  & $66$ &  & $2{\cdot} 10^{4}$ & $2{\cdot} 10^{4}$ &  &  &  & \\
\hline
\end{tabular}
\label{tab:max}
\end{table}
\begin{table}[t]
\caption{Results for benchmark POMDPs with minimisation objective}
\centering
\scriptsize
\begin{tabular}{|cc||r||r||r|r|r|r|r||r|}
\hline
\multicolumn{2}{|c||}{Benchmark} & \multicolumn{1}{c||}{Data} & \multicolumn{1}{c||}{\tool{Prism}} & \multicolumn{6}{c|}{\tool{Storm}}\\
\multicolumn{1}{|c}{Model} & \multicolumn{1}{c||}{$\phi$} & \multicolumn{1}{c||}{$\states$/$\actions$/$\observations$} & \multicolumn{1}{c||}{} & \multicolumn{1}{c|}{Cut-Off} & \multicolumn{4}{c||}{Cut-Off + Clipping} & \multicolumn{1}{c|}{Over-}\\
\multicolumn{1}{|c}{} & \multicolumn{1}{c||}{} & \multicolumn{1}{c||}{} & \multicolumn{1}{c||}{} &  \multicolumn{1}{c|}{Only} & \multicolumn{1}{c|}{$\resolution$=2} & \multicolumn{1}{c|}{$\resolution$=3} & \multicolumn{1}{c|}{$\resolution$=4} & \multicolumn{1}{c||}{$\resolution$=6} & \multicolumn{1}{c|}{Approx.}\\
\hline\hline
\model{Grid} & \multirow{3}{*}{$R_\mathrm{min}$} & $17$ & $[4.52, \mathbf{4.7}]$ & ${\le}\, 4.78$ & ${\le}\, 4.78$ & ${\le}\, 4.78$ & ${\le}\, 4.78$ & \multirow{3}{*}{TO} & ${\ge}\, 4.52$\\
\multirow{2}{*}{4-0.1} &  & $62$ & \ensuremath{649\text{s}} & \ensuremath{${<}\,$ 1\text{s}} & \ensuremath{15.6\text{s}} & \ensuremath{148\text{s}} & \ensuremath{1940\text{s}} &  & \\
 &  & $3$ & $\eta=$10 & $258$ & $255$ & $255$ & $255$ &  & \\
\hline
\model{Grid} & \multirow{3}{*}{$R_\mathrm{min}$} & $17$ & $[6.12, \mathbf{6.31}]$ & ${\le}\, 6.56$ & ${\le}\, 6.56$ & ${\le}\, 6.56$ & ${\le}\, 6.56$ & \multirow{3}{*}{TO} & ${\ge}\, 6.08$\\
\multirow{2}{*}{4-0.3} &  & $62$ & \ensuremath{1077\text{s}} & \ensuremath{${<}\,$ 1\text{s}} & \ensuremath{15.8\text{s}} & \ensuremath{148\text{s}} & \ensuremath{1983\text{s}} &  & \\
 &  & $3$ & $\eta=$10 & $255$ & $256$ & $256$ & $256$ &  & \\
\hline\hline
\model{Maze2} & \multirow{3}{*}{$R_\mathrm{min}$} & $15$ & $[6.32, \mathbf{6.32}]$ & ${\le}\, 6.34$ & ${\le}\, 6.34$ & ${\le}\, 6.34$ & ${\le}\, 6.34$ & ${\le}\, 6.34$ & ${\ge}\, 6.32$\\
\multirow{2}{*}{0.1} &  & $54$ & \ensuremath{1.79\text{s}} & \ensuremath{${<}\,$ 1\text{s}} & \ensuremath{${<}\,$ 1\text{s}} & \ensuremath{${<}\,$ 1\text{s}} & \ensuremath{${<}\,$ 1\text{s}} & \ensuremath{2.02\text{s}} & \\
 &  & $8$ & $\eta=$10 & $91$ & $90$ & $90$ & $90$ & $90$ & \\
\hline\hline
\model{Netw} & \multirow{3}{*}{$R_\mathrm{min}$} & $4589$ & $[3.17, \mathbf{3.2}]$ & ${\le}\, 6.56$ & ${\le}\, 6.56$ & ${\le}\, 6.56$ & ${\le}\, 6.56$ & ${\le}\, 6.56$ & ${\ge}\, 3.14$\\
\multirow{2}{*}{2-8-20} &  & $6973$ & \ensuremath{211\text{s}} & \ensuremath{${<}\,$ 1\text{s}} & \ensuremath{5.31\text{s}} & \ensuremath{17.2\text{s}} & \ensuremath{42.3\text{s}} & \ensuremath{167\text{s}} & \\
 &  & $1173$ & $\eta=$10 & $2{\cdot} 10^{4}$ & $2{\cdot} 10^{4}$ & $2{\cdot} 10^{4}$ & $3{\cdot} 10^{4}$ & $3{\cdot} 10^{4}$ & \\
\hline
\model{Netw} & \multirow{3}{*}{$R_\mathrm{min}$} & $2{\cdot} 10^{4}$ & $[5.61, \mathbf{6.79}]$ & ${\le}\, 11.9$ & ${\le}\, 11.9$ & ${\le}\, 11.9$ & ${\le}\, 11.9$ & \multirow{3}{*}{TO} & ${\ge}\, 6.13$\\
\multirow{2}{*}{3-8-20} &  & $3{\cdot} 10^{4}$ & \ensuremath{7133\text{s}} & \ensuremath{3.51\text{s}} & \ensuremath{214\text{s}} & \ensuremath{1372\text{s}} & \ensuremath{4910\text{s}} &  & \\
 &  & $2205$ & $\eta=$6 & $1{\cdot} 10^{5}$ & $2{\cdot} 10^{5}$ & $2{\cdot} 10^{5}$ & $2{\cdot} 10^{5}$ &  & \\
\hline\hline
\model{Rocks} & \multirow{3}{*}{$R_\mathrm{min}$} & $6553$ &  & ${\le}\, 38$ & ${\le}\, 38$ & ${\le}\, 38$ & ${\le}\, 20$ & ${\le}\, 21$ & ${\ge}\, 20$\\
\multirow{2}{*}{12} &  & $3{\cdot} 10^{4}$ & TO / MO & \ensuremath{1.39\text{s}} & \ensuremath{61.1\text{s}} & \ensuremath{138\text{s}} & \ensuremath{230\text{s}} & \ensuremath{532\text{s}} & \\
 &  & $1645$ &  & $3{\cdot} 10^{4}$ & $3{\cdot} 10^{4}$ & $3{\cdot} 10^{4}$ & $5{\cdot} 10^{4}$ & $6{\cdot} 10^{4}$ & \\
\hline
\model{Rocks} & \multirow{3}{*}{$R_\mathrm{min}$} & $1{\cdot} 10^{4}$ & & ${\le}\, 44$ & ${\le}\, 44$ & ${\le}\, 44$ & ${\le}\, 26$ & ${\le}\, 27$ & ${\ge}\, 26$\\
\multirow{2}{*}{16} &  & $5{\cdot} 10^{4}$ & TO / MO & \ensuremath{3.85\text{s}} & \ensuremath{114\text{s}} & \ensuremath{230\text{s}} & \ensuremath{399\text{s}} & \ensuremath{1062\text{s}} & \\
 &  & $2761$ &  & $4{\cdot} 10^{4}$ & $4{\cdot} 10^{4}$ & $4{\cdot} 10^{4}$ & $6{\cdot} 10^{4}$ & $1{\cdot} 10^{5}$ & \\
\hline
\end{tabular}
\label{tab:min}
\end{table}
\paragraph{Implementation details}
We integrated \Cref{alg:clippingMDP} in the probabilistic model checker \tool{Storm}\cite{storm} as an extension of the POMDP verification framework described in \cite{bork2020}.
Inputs are a POMDP---encoded either explicitly or using an extension of the \tool{Prism} language~\cite{norman2017}---and a property specification.
Internally, POMDPs and MDPs are represented using sparse matrices.
The implementation supports minimisation\footnote{For minimisation, the under-approximation yields \emph{upper bounds}.} and maximisation of reachability probabilities, reach-avoid probabilities (\ie the probability to avoid a set of bad state until a set of goal states is reached), and expected total rewards.
In a preprocessing step, functions $\underapprox$ and $\stateinfimum$ as considered in \Cref{alg:clippingMDP} are generated.
For $\underapprox$, we consider the function $\precompunderapprox^\sched$ as in \Cref{lem:precompValueFunc}, where $\sched$ is a memoryless observation-based policy given by a heuristic%
\footnote{The heuristic uses optimal values obtained on the fully observable underlying MDP.}.
For the function $\stateinfimum$, we apply standard MDP analysis on the underlying MDP.
When exploring the abstraction MDP $\clippingmdp{\pomdp}$,
our heuristic expands a belief iff $|\clippingstates| \le  | \states | \cdot \max_{\observation \in \observation} |\obsfunction^{-1}(\observation)|$, where $|\clippingstates|$ is the number of already explored beliefs and $|\obsfunction^{-1}(\observation)|$ is the number of POMDP states with observation $\observation$.
Belief clipping can either be disabled entirely, or we consider candidate sets $\candidatebeliefset \subseteq \beliefgrid{\resolution}$, where
$\beliefgrid{\resolution} \colonequals \set{\belief \in \beliefs \mid \forall \state \in \states : \belief(\state) \in \set{\nicefrac{i}{\resolution} \mid i \in \nn, 0 \leq i \leq \resolution}}$ forms a finite, regular \emph{grid} of beliefs with resolution $\resolution \in \nn \setminus \set{0}$. Grid beliefs $\belief \in \beliefgrid{\resolution}$ are always expanded.
Furthermore, we exclude clipping candidates $\tilde{\belief}$ with $\stateclippingvalue(\state) > 0$ for $\state$ with $\stateinfimum(\state) = -\infty$; clipping with such candidates is not useful as it induces a value of $-\infty$. 
Expected total rewards on fully observable MDPs are computed using \emph{Sound Value Iteration}~\cite{quatmann2018} with relative precision $10^{-6}$.
MILPs are solved using~\tool{Gurobi}~\cite{gurobi}.

\paragraph{Set-up}
We evaluate our under-approximation approach with cut-offs only and with enabled belief clipping procedure using grid resolutions $\resolution = 2,3,4,6$.
We consider the same POMDP benchmarks%
\footnote{Instances with a finite belief MDP that would be fully explored by our algorithm are omitted since the exact value can be obtained without approximation techniques.}
as in \cite{norman2017,bork2020}.
The POMDPs are scalable versions of case studies stemming from various application domains.
%
To establish an external baseline, we compare with the approach of~\cite{norman2017} implemented in \tool{Prism}~\cite{prism}.
\tool{Prism} generates an under-approximation based on an optimal policy for an over-approximative MDP which---in contrast to \tool{Storm}---means that always both, under- and over-approximations, have to be computed.
We ran \tool{Prism} with resolutions $\resolution = 2, 3, 4, 6, 8, 10$ and report on the \emph{best} approximation obtained.
To provide a further reference for the tightness of our under-approximation, we compute over-approximative bounds as in~\cite{bork2020} using the implementation in \tool{Storm} with a resolution of $\resolution = 8$.
All experiments were run on an Intel\textsuperscript{\textregistered} Xeon\textsuperscript{\textregistered} Platinum 8160 CPU using 4 threads%
\footnote{For our implementation, only \tool{Gurobi} runs multi-threaded. \tool{Prism} uses multiple threads for garbage collection.}%
, 64GB RAM and a time limit of 2~hours.

\begin{figure}[t]
\begin{tikzpicture}
			\begin{axis}[
			    width=\textwidth,
			    height=4.1cm,
 				xlabel={Number of explored beliefs $|\clippingstates|$},
 				xmin=0,
 				xmax = 31000,
 			 	xtick={0,10000,20000,30000,40000,50000},
 				xticklabel style={/pgf/number format/fixed, /pgf/number format/precision=5},
 				scaled x ticks=false,
 				ylabel={$\pr{}{}{\goalstates}$},
 				ymin=0.48,
 				ymax=1.02,
 				ytick={0.5,0.75,1},
 				axis on top,
 				legend style={at={(0.95,0.15)}, anchor=east, legend columns=-1}
 				]
				
 		      \addplot[mark=x, mark size=2pt, RWTHblue, thick, only marks ] table [x=x, y=y, col sep=comma] {drone_cutoff.csv};
 		      \addplot[mark=*, mark size=1.35pt, RWTHmagenta, thick, only marks ] table [x=x, y=y, col sep=comma] {drone_grid2.csv};
            \legend{Cut-Off,$\resolution = 2$}
 			  	\addplot[densely dashed] coordinates {(0,0.9745837231) (50000,0.9745837231)};
 			\end{axis} 
\end{tikzpicture}
\caption{Accuracy for \model{Drone} 4-2 with different sizes of approximation MDP $\clippingmdp{\pomdp}$}
\label{fig:drone}
\end{figure}

\paragraph{Results}
\Cref{tab:max,tab:min} show our results for maximising and minimising properties, respectively.
The first columns contain for each POMDP the benchmark name, model parameters, property type (probabilities (P) or rewards (R)), and the numbers of states, state-action pairs, and observations.
Column \tool{Prism} gives the result with the smallest gap between over- and under-approximation computed with the approach of~\cite{norman2017}. 
For maximising (minimising) properties, our approach competes with the lower (upper) bound of the provided interval. The relevant value is marked in bold. We also provide the computation time and the considered resolution $\resolution$.
For our implementation, we give results for the configuration with disabled clipping and for clipping with different resolutions $\resolution$.
In each cell, we give the obtained 
value, the computation time and the number of states in the abstraction MDP $\clippingmdp{\pomdp}$. Time- 
and memory-outs are indicated by TO and MO.
The right-most column indicates the over-approximation value computed via~\cite{bork2020}.

\paragraph{Discussion}
The pure cut-off approach yields valid under-approximations in \emph{all} benchmark instances---often exceeding the accuracy of the approach of~\cite{norman2017} while being consistently faster.
In some cases, the resulting values improve when clipping is enabled. 
However, larger candidate sets significantly increase the computation time which stems from the fact that many clipping MILPs have to be solved.

For \model{Drone} 4-2, \Cref{fig:drone} plots the resulting under-approximation values ($y$-axis) for varying sizes of the explored MDP $\clippingmdp{\pomdp}$ ($x$-axis).
The horizontal, dashed line indicates the computed over-approximation value.
The quality of the approximation further improves with an increased number of explored beliefs. 

\section{Conclusion}
We presented techniques to safely under-approximate expected total rewards in POMDPs.
The approach scales to large POMDPs and often produces tight lower bounds.
Belief clipping generally does not improve on the simpler cut-off approach in terms of results and performance. However, considering---and optimising---the approach for particular classes of POMDPs might prove beneficial.
Future work includes integrating the algorithm into a refinement loop that also considers over-approximation techniques from~\cite{bork2020}.
Furthermore, lifting our approach to partially observable stochastic games is promising.

\paragraph{Data Availability} The artifact~\cite{artifact} accompanying this paper contains source code, benchmark files, and replication scripts for our experiments.

\bibliography{main}
\bibliographystyle{splncs04}

\vfill

{\small\medskip\noindent{\bf Open Access} This chapter is licensed under the terms of the Creative Commons\break Attribution 4.0 International License (\url{http://creativecommons.org/licenses/by/4.0/}), which permits use, sharing, adaptation, distribution and reproduction in any medium or format, as long as you give appropriate credit to the original author(s) and the source, provide a link to the Creative Commons license and indicate if changes were made.}

{\small \spaceskip .28em plus .1em minus .1em The images or other third party material in this chapter are included in the chapter's Creative Commons license, unless indicated otherwise in a credit line to the material.~If material is not included in the chapter's Creative Commons license and your intended\break use is not permitted by statutory regulation or exceeds the permitted use, you will need to obtain permission directly from the copyright holder.}

\medskip\noindent\includegraphics{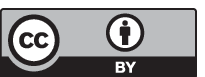}

 \clearpage
 \appendix
 \section{Dealing with unobservable goal states}\label{app:goal}
We argue that it can be assumed---without loss of generality---that goal states can be observed when computing expected total rewards.
More precisely, given an instance for \Cref{prob:pomdp}---\ie a POMDP $\pomdp = \pomdptuple$ with underlying MDP $\mdp = \mdptuple$, a reward structure $\rewards$, and a set of goal states $\goalstates \subseteq \states$---we may impose the following assumption:
\begin{assumption}\label{as:goalobs}
There is $\observations' \subseteq \observations$ such that $\state \in \goalstates$ iff $\obsof{\state} \in \observations'$.
\end{assumption}
To show that this assumption is indeed without loss of generality, suppose that it does \emph{not} hold for the above POMDP $\pomdp$.
Consider the POMDP $\pomdp' = \tuple{\mdp',\observations \uplus \set{\mathfrak{g}}, \obsfunction'}$, with underlying MDP $\mdp' = \tuple{\states \uplus \tilde{\goalstates}, \actions \uplus \set{\goalaction}, \transitions', \sinit[']}$, where
\begin{itemize}
\item $\tilde{\goalstates} = \set{\tilde{\state} \mid \state \in \goalstates}$
    \item $\sinit['] = \widetilde{\sinit} \in \tilde{\goalstates}$ if $\sinit \in \goalstates$ and $\sinit['] = \sinit$ otherwise,
\end{itemize}
 and for $\state,\state' \in \states$, $\tilde{\state} \in \tilde{\goalstates}$, and $\action \in \actions$:
\begin{itemize}
\item $\transitions'(\state,\action,\state') = [\state' \notin \goalstates] \cdot \transitions(\state,\action,\state')$,
\item $\transitions'(\state,\action,\tilde{\state}) = \transitions(\state,\action,\state')$,
\item $\transitions(\tilde{\state}, \goalaction, \tilde{\state}) = 1$, and $\transitions(\cdot, \cdot, \cdot) = 0$ in all other cases,
\item $\obsfunction'(\state) = \obsof{\state}$, and $\obsfunction'(\tilde{\state}) = \mathfrak{g}$.
\end{itemize}
Intuitively, $\pomdp'$ is obtained from $\pomdp$ by creating a copy $\tilde{\state} \in \tilde{\goalstates}$ for all goal states $\state \in \goalstates$ with a new observation $\obsfunction'(\tilde{\state}) = \mathfrak{g}$ and redirecting all incoming transitions of  $\state \in \goalstates$ to $\tilde{\state}$.
We also adapt the reward structure $\rewards$ for $\pomdp$, yielding reward structure $\rewards'$ for $\pomdp'$, where for $\state, \state' \in \states$, $\tilde{\state} \in \tilde{\goalstates}$, and $\action \in \actions$:
\begin{itemize}
    \item $\rewards'(\state,\action,\state') = [\state' \notin \goalstates] \cdot \rewards(\state,\action,\state')$,
    \item $\rewards'(\state,\action, \tilde{\state}) = \rewards(\state,\action,\state')$, and
    \item $\rewards'(\cdot,\cdot,\cdot) = 0$ in all other cases.
\end{itemize}
The POMDP $\pomdp'$, reward structure $\rewards'$ and goal state set $\tilde{\goalstates}$ satisfy \Cref{as:goalobs}.

To relate the values obtained for $\pomdp$ and $\pomdp'$, we consider a mapping $f \colon \scheds{\pomdp} \to \scheds{\pomdp'}$ such that for $\sched\in \scheds{\pomdp}$ and $\finpath' \in \finpaths{\pomdp'}$ we have
\[
f(\sched)(\finpath') = 
\begin{cases}
\sched(\finpath') & \text{if } \finpath \text{ does not visit a state} \tilde{\state} \in \tilde{\goalstates}\\
\set{\goalaction \mapsto 1} & \text{otherwise.}
\end{cases}
\]
The policy $f(\sched)$ is well-defined since
\begin{itemize}
    \item paths of $\pomdp'$ that do not visit $\tilde{\goalstates}$ are also available in $\pomdp$ and
    \item once a path $\finpath'$ of $\pomdp'$ reaches a state $\tilde{\state} \in \tilde{\goalstates}$, the state can not be left again---implying $\last{\finpath'} = \tilde{\state}$ and thus the action $\goalaction$ is always enabled.
\end{itemize}
The mapping $f$ is surjective\footnote{However, $f$ is---in general---not injective: two policies $\sched_1, \sched_2 \in f^{-1}(\sched')$ might (only) differ for paths that have already visited a goal state.} \ie the inverse $f^{-1} \colon \scheds{\pomdp'} \to 2^{(\scheds{\pomdp})}$ with $f^{-1}(\sched') = \set{\sched \in \scheds{\pomdp} \mid f(\sched) = \sched'}$ satisfies $|f^{-1}(\sched')| > 0$.
Moreover, if $\sched$ is observation-based, then $f(\sched)$ is also observation-based.
Thus, the following lemma yields that we can solve \Cref{prob:pomdp} using $\pomdp'$, $\rewards'$ and $\tilde{\goalstates}$ instead of $\pomdp$, $\rewards$, and $\goalstates$.
\begin{lemma}\label{lem:goal}
For all policies $\sched \in \scheds{\pomdp}$ we have $\exprew{\pomdp,\rewards}{\sched}{\goalstates} = \exprew{\pomdp',\rewards'}{f(\sched)}{\tilde{\goalstates}}$.
\end{lemma}
\begin{proof}
Consider the sets of paths of $\pomdp'$:
$$\Pi'_\mathfrak{g} \colonequals \set{\infpath' \in \infpaths{\pomdp'} \mid \infpath' \text{ visits some state } \tilde{\state} \in \tilde{\goalstates}} $$
and $\Pi'_{\neg \mathfrak{g}} \colonequals \infpaths{\pomdp'} \setminus \Pi'_\mathfrak{g}$.

For $\infpath' = \state_0 \action_1  \dots \action_{n-1} \state_{n-1} \action_{n} \tilde{\state_n}\, \goalaction\, \tilde{\state_n}\, \goalaction \dots  \in \Pi'_\mathfrak{g}$, we define the set
$$\Lambda(\infpath') \colonequals \set{\infpath \in \infpaths{\pomdp} \mid \infpath[n] = \state_0 \action_1  \dots \action_{n-1} \state_{n-1} \action_{n} \state_n \text{ with } \state_n \in \goalstates}.$$
For all $\infpath \in \Lambda(\infpath')$ it holds that
\begin{align*}
\rew{\pomdp',\rewards',\tilde{\goalstates}}(\infpath') 
~=~ \sum_{i=1}^n \rewards(\state_{i-1},\action_i,\state_i) ~=~\rew{\pomdp,\rewards,\goalstates}(\infpath).
\end{align*}
Furthermore, for $\sched \in \scheds{\pomdp}$ we have 
\begin{align*}
    \probmeasure{\mdp'}^{f(\sched), \state_0}(\set{\infpath'})
    = \prod_{i=1}^n  \sched(\infpath'[i-1])(\action_i) \cdot \transitions(\state_{i-1},\action_i,\state_i)
    = \probmeasure{\mdp}^{\sched, \state_0}(\Lambda(\infpath')).
\end{align*}

We observe that the set $\Pi'_\mathfrak{g}$ is countable and that
\[
\infpaths{\pomdp} = \Pi'_{\neg \mathfrak{g}} ~\uplus~ \biguplus_{\infpath' \in \Pi'_{\mathfrak{g}}} \Lambda(\infpath').
\]
For $\sched \in \scheds{\pomdp}$ the following holds which concludes the proof:
\begin{align*}
&~\exprew{\pomdp, \rewards}{\sched}{\goalstates} \\
~=&~ \int_{\infpath \in \infpaths{\pomdp}} \rew{\pomdp, \rewards, \goalstates}(\infpath) \cdot \probmeasure{\pomdp}^{\sched}(d\infpath)\\
~=&~ \int_{\infpath \in \Pi'_{\neg \mathfrak{g}}} \rew{\pomdp, \rewards, \goalstates}(\infpath) \cdot \probmeasure{\pomdp}^{\sched}(d\infpath)~+~ \sum_{\infpath' \in \Pi'_{\mathfrak{g}}} \int_{\infpath \in \Lambda(\infpath')} \rew{\pomdp, \rewards, \goalstates}(\infpath) \cdot \probmeasure{\pomdp}^{\sched}(d\infpath)\\
~=&~ \int_{\infpath \in \Pi'_{\neg \mathfrak{g}}} \rew{\pomdp, \rewards, \goalstates}(\infpath) \cdot \probmeasure{\pomdp}^{\sched}(d\infpath)~+~ \sum_{\infpath' \in \Pi'_{\mathfrak{g}}} \rew{\pomdp', \rewards', \tilde{\goalstates}}(\infpath') \cdot  \int_{\infpath \in \Lambda(\infpath')} \probmeasure{\pomdp}^{\sched}(d\infpath)\\
~=&~ \int_{\infpath \in \Pi'_{\neg \mathfrak{g}}} \rew{\pomdp, \rewards, \goalstates}(\infpath) \cdot \probmeasure{\pomdp}^{\sched}(d\infpath)~+~ \sum_{\infpath' \in \Pi'_{\mathfrak{g}}} \rew{\pomdp', \rewards', \tilde{\goalstates}}(\infpath') \cdot \probmeasure{\pomdp}^{\sched}(\Lambda(\infpath'))\\
~=&~ \int_{\infpath' \in \Pi'_{\neg \mathfrak{g}}} \rew{\pomdp', \rewards', \tilde{\goalstates}}(\infpath') \cdot \probmeasure{\pomdp'}^{f(\sched)}(d\infpath')~+~ \sum_{\infpath' \in \Pi'_{\mathfrak{g}}} \rew{\pomdp', \rewards', \tilde{\goalstates}}(\infpath') \cdot \probmeasure{\pomdp}^{f(\sched)}(\set{\infpath'})\\
~=&~ \int_{\infpath' \in \infpaths{\pomdp'}} \rew{\pomdp', \rewards', \tilde{\goalstates}}(\infpath') \cdot \probmeasure{\pomdp'}^{f(\sched)}(d\infpath') ~=~ \exprew{\pomdp', \rewards'}{f(\sched)}{\tilde{\goalstates}}.
\end{align*}
\qed
\end{proof}

\section{Proofs}\label{app:proof}

We provide proofs for our main results.
We first show an auxiliary lemma that will be helpful for proving some of our claims.

In the following, we slightly abuse notations by considering observation-based policies as functions on observation traces (instead of paths). For observation-based policies this is unambiguous as they behave the same on different paths $\finpath_1, \finpath_2$ with the same observation trace $\obsof{\finpath_1} = \obsof{\finpath_2}$.

For an observation-based policy $\sched \in \obsscheds{\pomdp}$, $\observation \in \observations$, and $\action \in \actions$ let $\schednext$ be the (obsevation-based) policy, where for observation trace $\observation_0 \action_1   \hdots \action_n \observation_n$ we set 
$$\schednext(\observation_0 \action_1   \hdots \action_n \observation_n) \colonequals \sched(\observation \action\observation_0 \action_1   \hdots \action_n \observation_n).$$
In other words, $\schednext$ behaves as $\sched$ after it has observed  $\observation$ and $\action$. 
We consider the $n$-step value function under an observation-based policy $\sched \in \obsscheds{\pomdp}$ which is given by $\valuefunc[n]^{\sched}: \beliefs_\pomdp \to \rr$ , where for $\belief \in \beliefs_\pomdp$ with $\observation = \obsof{\belief}$ we set
\begin{align*}
\valuefunc[0]^{\sched}(\belief) & \colonequals 0 \\
\valuefunc[n]^{\sched}(\belief) & \colonequals \begin{dcases}
0 & \text{if } \belief \in \goalbels \\
\sum_{\action \in \actions} \sched(\observation)(\action) \; \smashoperator[lr]{\sum_{\belief' \in \post[\beliefmdp{\pomdp}]{\belief}{\action}}} \; \belieftransitions(\belief, \action, \belief') \cdot (\beliefrewards(\belief,\action,\belief')+\valuefunc[n-1]^{\schednext}(\belief')) & \text{otherwise}
\end{dcases}
\end{align*}
The (optimal) value function under $\sched$ is given by $\optvaluefunc_\sched \colon \beliefs_\pomdp \mapsto \rr^\infty$ with $$\optvaluefunc_\sched(\belief) \colonequals \lim_{n \to \infty} \valuefunc[n]^\sched(\belief).$$

\begin{lemma}\label{lem:valsched}
For all $\sched \in \obsscheds{\pomdp}$ and $\belief \in \beliefs_\pomdp$ we have
$$ \optvaluefunc_\sched(\belief) ~=~ \sum_{\state \in \supp{\belief}} \belief(\state) \cdot \exprew[\state]{\pomdp,\rewards}{\sched}{\goalstates}.$$
\end{lemma}
\begin{proof}
Let $\observation \colonequals \obsof{\belief}$.
The proof considers the characterization of $\exprew[\state]{\pomdp,\rewards}{\sched}{\goalstates}$ for POMDP state $\state \in  \supp{\belief} \subseteq \states$ using Bellman equations: for $n \in \nn$, let $W_n^\sched \colon \states \to \rr$ be given by
\begin{align*}
W_0^{\sched}(\state) & \colonequals 0 \\
W_n^\sched(\state) & \colonequals \begin{dcases}
0 & \text{if } \state \in \goalstates \\
\sum_{\action \in \actions} \sched(\observation)(\action) \; \sum_{\state' \in \states} \transitions(\state, \action, \state') \cdot (\rewards(\state,\action,\state') \cdot W_{n-1}^{\schednext}(\state')) & \text{otherwise.}
\end{dcases}
\end{align*}
It is well-known (e.g., \cite{puterman1994}) that $\exprew[\state]{\pomdp,\rewards}{\sched}{\goalstates} = \lim_{n\to\infty} W_n^\sched(\state)$. Thus,
\begin{align*}
     &~\optvaluefunc_\sched(\belief) ~=~ \sum_{\state \in \supp{\belief}} \belief(\state) \cdot \exprew[\state]{\pomdp,\rewards}{\sched}{\goalstates}\\
     \iff&~\lim_{n \to \infty} \valuefunc[n]^\sched(\belief) ~=~ \sum_{\state \in \supp{\belief}} \belief(\state) \cdot \lim_{n\to\infty} W_n^\sched(\state)\\
    \iff&~\lim_{n \to \infty} \valuefunc[n]^\sched(\belief) ~=~ \lim_{n\to\infty} \sum_{\state \in \supp{\belief}} \belief(\state) \cdot  W_n^\sched(\state).
\end{align*}
We prove the lemma by showing that for all $\sched \in \obsscheds{\pomdp}$ and $n \in \nn$ we have $\valuefunc[n]^\sched(\belief) =  \sum_{\state \in \supp{\belief}} \belief(\state) \cdot  W_n^\sched(\state)$. This is shown by induction.

\paragraph{Base case ($n=0$)}
We have that $\valuefunc[0]^{\sched}(\belief) = 0 = \sum_{\state \in \states} \belief(\state) \cdot W_0^\sched(\state)$.

\paragraph{Induction hypothesis} For fixed $n$ and all $\sched \in \obsscheds{\pomdp}$ we have
$\valuefunc[n]^\sched(\belief) =  \sum_{\state \in \supp{\belief}} \belief(\state) \cdot  W_n^\sched(\state)$.

\paragraph{Induction step ($n \to n+1$)}
If $\belief \in \goalbels$, we have $\supp{\belief} \subseteq \goalstates$ and thus
$$\valuefunc[n+1]^{\sched}(\belief) = 0 = \sum_{\state \in \states} \belief(\state) \cdot W_{n+1}^\sched(\state).$$
For $\belief \notin \goalbels$ we get:
\begin{align*}
&\hphantom{=} \valuefunc[n+1]^{\sched}(\belief) \\
& = \sum_{\action \in \actions} \sched(\observation)(\action) \; \smashoperator[lr]{\sum_{\belief' \in \post[\beliefmdp{\pomdp}]{\belief}{\action}}} \; \belieftransitions(\belief, \action, \belief') \cdot (\beliefrewards(\belief,\action,\belief')+\valuefunc[n]^{\schednext}(\belief'))\\
& = \sum_{\action \in \actions} \sched(\observation)(\action) \sum_{\observation' \in \observations} \transitions(\belief, \action, \observation') \cdot (\beliefrewards(\belief,\action,\nextbelief{\belief}{\action}{\observation'})+\valuefunc[n]^{\schednext}(\nextbelief{\belief}{\action}{\observation'}))\\
&\overset{(IH)}{=} \sum_{\action \in \actions} \sched(\observation)(\action) \sum_{\observation' \in \observations} \transitions(\belief, \action, \observation') \cdot (\beliefrewards(\belief,\action,\nextbelief{\belief}{\action}{\observation'})+ \sum_{\state' \in \states}\nextbelief{\belief}{\action}{\observation'}(\state') \cdot W_n^{\schednext}(\state'))\\
& = \sum_{\action \in \actions} \sched(\observation)(\action) \sum_{\observation' \in \observations} \transitions(\belief, \action, \observation') \cdot \left(\frac{\sum_{\state \in \states} \belief(\state) \cdot \sum_{\state' \in \states}[\obsof{s'} = \observation'] \cdot \rewards(\state,\action,\state')\cdot \transitions(\state,\action,\state')}{\transitions(\belief,\action,\observation')} \right. \\
& \left. \hphantom{\sum_{\action \in \actions} \sched(\observation)(\action) \sum_{\observation' \in \observations} \transitions(\belief, \action, \observation') \cdot \left(\right.} + \sum_{\state' \in \states} \frac{[\obsof{\state'} = \observation'] \cdot \sum_{\state \in \states} \belief(\state) \cdot \transitions(\state,\action,\state')}{\transitions(\belief,\action,\observation')} \cdot W_n^{\schednext}(\state')\right)\\
& = \sum_{\action \in \actions} \sched(\observation)(\action) \sum_{\observation' \in \observations} \left(\sum_{\state \in \states} \belief(\state) \cdot \sum_{\state' \in \states}[\obsof{s'} = \observation'] \cdot \rewards(\state,\action,\state')\cdot \transitions(\state,\action,\state')\right. \\
& \left. \hphantom{\sum_{\action \in \actions} \sched(\observation)(\action) \sum_{\observation' \in \observations} \left(\sum_{\state \in \states}\right.} + \sum_{\state' \in \states} [\obsof{\state'} = \observation'] \cdot \sum_{\state \in \states} \belief(\state) \cdot \transitions(\state,\action,\state') \cdot W_n^{\schednext}(\state')\right)\\
& = \sum_{\action \in \actions} \sched(\observation)(\action) \sum_{\observation' \in \observations} \sum_{\state \in \states} \sum_{\state' \in \states}\belief(\state) \cdot [\obsof{s'} = \observation'] \cdot \transitions(\state,\action,\state') \cdot (\rewards(\state,\action,\state') + W_n^{\schednext}(\state'))\\
& = \sum_{\action \in \actions} \sched(\observation)(\action) \sum_{\state \in \states} \sum_{\state' \in \states}\belief(\state) \cdot \transitions(\state,\action,\state') \cdot (\rewards(\state,\action,\state') + W_n^{\schednext}(\state'))\\\\
& = \sum_{\action \in \actions} \sched(\observation)(\action) \sum_{\state \in \states} \belief(\state) \cdot \sum_{\state' \in \states}\transitions(\state,\action,\state') \cdot (\rewards(\state,\action,\state') + W_n^{\schednext}(\state'))\\
& = \sum_{\state \in \states} \belief(\state) \cdot W_{n+1}^\sched(\state).
\end{align*}
\qed
\end{proof}

\subsection{Proof of \Cref{lem:precompValueFunc}}
\lemPrecompValueFunc*
\begin{proof}
Using \Cref{lem:valsched} we get
\begin{align*}
\underapprox(\belief) &= \optvaluefunc_\sched(\belief) \le \sup_{\sched' \in \obsscheds{\pomdp}} \optvaluefunc_{\sched'}(\belief) = \optvaluefunc(\belief).
\end{align*}
\qed
\end{proof}

\subsection{Proof of \Cref{thm:MILP}}
\thmMILP*
\begin{proof}
Let $f$ be an optimal solution such that for some $\tilde{\belief} \in \candidatebeliefset$ we have $f(\beliefselectvar{\tilde{\belief}}) = 1$ and $f(\clippingvalue[]) < 1$. Due to \Cref{eq:exactlyOne} we have that $\sum_{\belief' \in \candidatebeliefset} f(\beliefselectvar{\belief'}) = 1$ and therefore $f(\beliefselectvar{\belief'}) = 0$ for all $\belief' \in \candidatebeliefset\setminus\{\tilde{\belief}\}$. 
For $\belief' \in \candidatebeliefset\setminus\{\tilde{\belief}\}$, \Cref{eq:linClip} reduces to 
\begin{align*}
   && f(\stateclippingvalue[\state]) & \geq \underbrace{\belief(\state)}_{\leq 1} - \underbrace{(1-f(\clippingvalue[]))}_{\ge 0} \cdot \underbrace{\belief'(\state)}_{\ge 0} - (1 - 0).
\end{align*}
As the right hand side of this constraint is at most 0, it is already implied by \Cref{eq:stateclipbnds}, \ie \Cref{eq:linClip} does not further constrain the value of $\stateclippingvalue[\state]$.
For $\tilde{\belief}$ and \Cref{eq:linClip}, we get:
\begin{align*}
f(\stateclippingvalue[\state]) & \geq \belief(\state) - (1-f(\clippingvalue[])) \cdot \tilde{\belief}(\state).
\end{align*}
As we minimise, we know that either 
\begin{itemize}
    \item $f(\stateclippingvalue[\state])\geq 0$ or
    \item $f(\stateclippingvalue[\state]) \geq \belief(\state) - (1-f(\clippingvalue[])) \cdot \tilde{\belief}(\state)$
\end{itemize}
or both inequalities must hold with equality.
We show that $f(\stateclippingvalue[\state]) = \belief(\state) - (1-f(\clippingvalue[]))  \cdot \tilde{\belief}(\state)$ must hold for all states $\state \in \supp{\belief}$.
Towards a contradiction, assume there is at least one state $\state \in \supp{\belief}$ for which the constraint does \emph{not} hold with equality, \ie
\begin{equation*}
 \tag{I} f(\stateclippingvalue[\state]) >  \belief(\state) - (1-f(\clippingvalue[]))  \cdot \tilde{\belief}(\state)
\end{equation*}
and for all $\state' \neq \state$:
\begin{equation*}
\tag{II} f(\stateclippingvalue[\state']) \geq  \belief(\state') - (1-f(\clippingvalue[]))  \cdot \tilde{\belief}(\state')
\end{equation*}
From (I) and (II), we get
\begin{align*}
&&f(\clippingvalue[]) & = \sum_{\state' \in \supp{\belief}}f(\stateclippingvalue[\state']) &&\qquad\text{(Constr. \ref{eq:sum})}\\ 
&& & > \sum_{\state' \in \supp{\belief}}\Big( \belief(\state) - (1-f(\clippingvalue[]))  \cdot \tilde{\belief}(\state)\Big) &&\qquad\text{(I, II)}\\
&& & = \underbrace{\sum_{\state' \in \supp{\belief}}\belief(\state')}_{=1} ~-~ (1-f(\clippingvalue[])) ~ \cdot  \underbrace{\sum_{\state' \in \supp{\belief}}\tilde{\belief}(\state')}_{=1}  \\
&& & = 1 - (1-f(\clippingvalue[]))  \cdot 1 ~=~  f(\clippingvalue[]) \\
\implies && f(\clippingvalue[]) & > f(\clippingvalue[]) \quad \mathbf{\lightning}
\end{align*}
As the assumption leads to a contradiction, we know that for all $\state \in \supp{\belief}$, it holds that
\begin{equation*}
f(\stateclippingvalue[\state]) =\belief(\state) - (1-f(\clippingvalue[]))  \cdot \tilde{\belief}(\state).
\end{equation*}
We get 
\begin{align*}
    && f(\stateclippingvalue[\state]) & = \belief(\state) - (1-f(\clippingvalue[]))  \cdot \tilde{\belief}(\state)\\
    \iff && (1-f(\clippingvalue[])) \cdot \tilde{\belief}(\state) & = \belief(\state) - f(\stateclippingvalue[\state])\\
    \iff && \tilde{\belief}(\state) & = \frac{\belief(\state) - f(\stateclippingvalue[\state])}{1-f(\clippingvalue[])}. && \qquad (f(\clippingvalue[]) < 1)
\end{align*}
Thus, $\tilde{\belief}$ is an adequate clipping candidate for $\belief$ with clipping value $f(\clippingvalue[])$. As we minimise $f(\clippingvalue[])$, the clipping value is minimal, \ie $f(\clippingvalue[]) = \clippingvalue$.

For the other direction, let $\tilde{\belief}$ be an adequate clipping candidate for $\belief$ with minimal clipping value, \ie there are state clipping values given by a belief clip $\stateclippingvalue \in \mathcal{C}(\belief,\tilde{\belief})$ with $\tilde{\belief} = (\belief \ominus \stateclippingvalue)$. Furthermore, for all $\belief' \in \candidatebeliefset \setminus \set{\tilde{\belief}}$ we have that either $\belief'$ is not adequate or $\clippingvalue[\belief \to \belief'] \ge \clippingvalue$.
We claim that $f$ is an optimal solution to the MILP, where
\begin{itemize}
\item  $f(\beliefselectvar{\tilde{\belief}}) \colonequals 1$ and  $f(\beliefselectvar{\belief'}) \colonequals 0$ for all $\belief' \in \candidatebeliefset \setminus \set{\tilde{\belief}}$,
\item $f(\clippingvalue[]) \colonequals \clippingvalue$, and
\item $f(\stateclippingvalue[\state]) \colonequals \stateclippingvalue(\state)$ for all $\state \in \supp{\belief}$.
\end{itemize}
$f$ is a feasible solution as we can verify that \Crefrange{eq:exactlyOne}{eq:linClip} hold.
To show that $f$ is also an optimal solution, assume towards a contradiction that another solution $f'$ with $f'(\clippingvalue[]) < f(\clippingvalue[]$ exists. We have already argued above that for $\belief' \in \candidatebeliefset$ with $f'(\beliefselectvar{\belief'}) = 1$ this yields
\[
\clippingvalue[\belief \to \belief']~=~f'(\clippingvalue[]) < f(\clippingvalue[])~=~\clippingvalue
\]
which contradicts the assumption that the clipping value for $\tilde{\belief}$ is minimal.
\qed
\end{proof}

\subsection{Proof of \Cref{lem:clip}}\label{app:clip}
Given two beliefs $\belief$ and $\tilde{\belief}$ with the same observation, recall from \Cref{def:beliefclip,def:clippingvalue} that $\forall \state \in \supp{\belief}\colon \stateclippingvalue(\state) \in [0,\belief(s)]$ with $\clippingvalue \colonequals \sum_{\state \in \supp{\belief}} \stateclippingvalue(\state) < 1 $ and $\forall \state \in \supp{\belief}$:
\[
\frac{\belief(\state) - \stateclippingvalue(\state)}{1-\clippingvalue} = \tilde{\belief}(\state)
\quad \iff \quad
\belief(\state) = (1-\clippingvalue) \cdot \tilde{\belief}(\state) + \stateclippingvalue(\state).
\]
Furthermore, $
\forall \state \in \supp{\belief}\colon\, \stateinfimum(\state) = \exprew[\state]{\mdp, \rewards}{\min}{\goalstates}$.
In particular, $\stateinfimum(\state) \le \exprew[\state]{\mdp, \rewards}{\sched}{\goalstates}$ holds for all policies $\sched \in \obsscheds{\pomdp}$.
\lemClip*
\begin{proof}
Using \Cref{lem:valsched}, we get for every observation-based policy $\sched \in \obsscheds{\pomdp}$ that:
\begin{align*}
    \optvaluefunc(\belief) &\ge \optvaluefunc_\sched(\belief) \\
    &= \smashoperator[r]{\sum_{\state \in \supp{\belief}}} \belief(\state) \cdot \exprew[\state]{\pomdp,\rewards}{\sched}{\goalstates} \\ 
    &= \smashoperator[r]{\sum_{\state \in \supp{\belief}}} ((1-\clippingvalue) \cdot \tilde{\belief}(\state) + \stateclippingvalue(\state)) \cdot \exprew[\state]{\pomdp,\rewards}{\sched}{\goalstates} \\
    &= (1-\clippingvalue) \cdot \Big(\smashoperator[r]{\sum_{\state \in \supp{\belief}}} \tilde{\belief}(\state) \cdot \exprew[\state]{\pomdp,\rewards}{\sched}{\goalstates}\Big) + \\
    & \qquad \qquad \Big(\smashoperator[r]{\sum_{\state \in \supp{\belief}}} \stateclippingvalue(\state) \cdot \exprew[\state]{\pomdp,\rewards}{\sched}{\goalstates}\Big)\\
    &= (1-\clippingvalue) \cdot \optvaluefunc_\sched(\tilde{\belief}) + \smashoperator[r]{\sum_{\state \in \supp{\belief}}} \stateclippingvalue(\state) \cdot \exprew[\state]{\pomdp,\rewards}{\sched}{\goalstates}\\
    &\ge (1-\clippingvalue) \cdot \optvaluefunc_\sched(\tilde{\belief}) + \smashoperator[r]{\sum_{\state \in \supp{\belief}}} \stateclippingvalue(\state) \cdot \stateinfimum(\state).
\end{align*}
Therefore,
\begin{align*}
\optvaluefunc(\belief) 
&\ge \sup_{\sched \in \obsscheds{\pomdp}} \Big( (1-\clippingvalue) \cdot \optvaluefunc_\sched(\tilde{\belief}) + \smashoperator[r]{\sum_{\state \in \supp{\belief}}} \stateclippingvalue(\state) \cdot \stateinfimum(\state)\Big)\\
&\ge (1-\clippingvalue) \cdot \sup_{\sched \in \obsscheds{\pomdp}} \optvaluefunc_\sched(\tilde{\belief}) + \smashoperator[r]{\sum_{\state \in \supp{\belief}}} \stateclippingvalue(\state) \cdot \stateinfimum(\state)\\
&\ge (1-\clippingvalue) \cdot \optvaluefunc(\tilde{\belief}) + \smashoperator[r]{\sum_{\state \in \supp{\belief}}} \stateclippingvalue(\state) \cdot \stateinfimum(\state).
\end{align*}
\qed
\end{proof}

\subsection{Proof of \Cref{thm:clippingUnderapprox}}
We first argue that the following $n$-step variant of \Cref{lem:clip} holds for $n \in \nn$.
Let the minimal $n$-step value function on the underlying MDP $\mdp$ be given by $\stateinfimum_n \colon \states \to \rr$ with $\stateinfimum_0(\state)  \colonequals 0$ and for $n > 0$:
\begin{align*}
\stateinfimum_n(\state) \colonequals \begin{dcases}
0 & \text{if } \state \in \goalstates \\
\min_{\action \in \actions(\state)} \sched(\observation)(\action) \; \sum_{\state' \in \states} \transitions(\state, \action, \state') \cdot (\rewards(\state,\action,\state') \cdot \stateinfimum_{n-1}(\state')) & \text{otherwise.}
\end{dcases}
\end{align*}
\begin{lemma}\label{lem:clipn}
For belief $\belief$ and adequate clipping candidate $\tilde{\belief}$ with clipping values $\clippingvalue$ and $\stateclippingvalue(\state)$ ($\state \in \supp{\belief}$) we have
$$
(1-\clippingvalue) \cdot \valuefunc[n](\tilde{\belief}) + \clippingvalue \cdot  \sum_{\state \in \supp{\belief}} \frac{\stateclippingvalue(\state)}{\clippingvalue} \cdot \stateinfimum_{n}(\state) ~\le~ \valuefunc[n](\belief).
$$
\end{lemma}
\begin{proof}
Recall from the proof of \Cref{lem:valsched} that
$$\valuefunc[n]^\sched(\belief) =  \sum_{\state \in \supp{\belief}} \belief(\state) \cdot  W_n^\sched(\state).$$
Furthermore, we have $\stateinfimum_n(\state) \le W_n^\sched(\state)$ for all $\sched$.
With this in mind, the proof is analogous to the proof of \Cref{lem:clip}.
\qed
\end{proof}

\clippingUnderapprox*
\begin{proof}
Let $\clippingmdp{\pomdp}$ and $\clippingrewards$ be the clipping belief MDP and reward structure returned by \Cref{alg:clippingMDP}.
As mentioned in \Cref{sec:algorithm}, we assume that there are no ``$\clippingaction$-cycles'', \ie paths of the form $\belief_1\, \clippingaction\, \dots\, \clippingaction \,\belief_n\, \clippingaction\, \belief_1$ in $\clippingmdp{\pomdp}$.

Since $\clippingmdp{\pomdp}$ is a finite MDP, it suffices to show
$$
\exprew[\belief]{\clippingmdp{\pomdp},\clippingrewards}{\sched}{\cutgoals} ~\le~ \optvaluefunc(\belief)
$$
for every $\belief \in \clippingstates \setminus \set{\bcut}$ and every memoryless, deterministic policy $\sched \colon \clippingstates \to \actions \uplus \set{\goalaction, \cutaction, \clippingaction}$.
We now fix such a policy $\sched$ for $\clippingmdp{\pomdp}$ and define the function $\mathcal{V}^\sched_n \colon \beliefs_\pomdp \to \rr^\infty$ for $n \in \nn$ as follows.
We set $\mathcal{V}^\sched_0(\belief) \colonequals 0$. For $n > 0$, distinguish the following cases for $\belief \in \clippingstates \setminus \set{\bcut}$.
\begin{itemize}
    \item If $\sched(\belief) = \action \in \actions$, then $\mathcal{V}^\sched_n(\belief) \colonequals \sum_{\belief' \in \clippingstates} \belieftransitions(\belief, \action, \belief') \cdot (\beliefrewards(\belief, \action, \belief') \cdot \mathcal{V}^\sched_{n-1}(\belief'))$.
    \item If $\sched(\belief) = \goalaction$ (\ie $\belief \in \goalbels$), then $\mathcal{V}^\sched_n(\belief) \colonequals  0$.
    \item If $\sched(\belief) = \cutaction$, then $\mathcal{V}^\sched_n(\belief) \colonequals \min(\underapprox(\belief), \valuefunc[n](\belief))$.
    \item If $\sched(\belief) = \clippingaction$, then there is a clipping candidate $\tilde{\belief}$  with clipping values $\clippingvalue$ and $\stateclippingvalue(\state)$ ($\state \in \supp{\belief}$). We set $\mathcal{V}^\sched_n(\belief) \colonequals (1-\clippingvalue) \cdot \mathcal{V}^\sched_n(\tilde{\belief}) + \clippingvalue \cdot \sum_{\state \in \supp{\belief}} \nicefrac{\stateclippingvalue(\state)}{\clippingvalue} \cdot \stateinfimum_{n}(\state)$.
    Due to the absence of $\clippingaction$-cycles, this case is well-defined---even though the step counter $n$ is not decremented.
\end{itemize}
Let $\mathcal{V}^*_\sched(\belief) \colonequals \lim_{n\to\infty} \mathcal{V}^\sched_n(\belief)$.
Note that for a belief $\belief$ with $\sched(\belief) = \cutaction$ we get $$\mathcal{V}^*_\sched(\belief) = \lim_{n\to\infty} \min(\underapprox(\belief), \valuefunc[n](\belief)) = \min(\underapprox(\belief), \optvaluefunc(\belief)) = \underapprox(\belief).$$
By construction of $\clippingmdp{\pomdp}$, it follows that $\mathcal{V}^*_\sched$ provides a solution to the classical Bellman equations for $\clippingmdp{\pomdp}$, \ie we have for all $\belief \in \clippingstates \setminus \set{\bcut}$:
\[
\mathcal{V}^*_\sched(\belief) = \exprew[\belief]{\clippingmdp{\pomdp},\clippingrewards}{\sched}{\cutgoals}.
\]
To prove \Cref{thm:clippingUnderapprox}, it thus suffices to show that $\lim_{n \to \infty} \mathcal{V}^\sched_n(\belief) \le \lim_{n \to \infty} \valuefunc[n](\belief)$.
We conclude the proof by showing that for all $n \in \nn$ and $\belief \in \clippingstates \setminus \set{\bcut}$ we have $\mathcal{V}^\sched_n(\belief) \le \valuefunc[n](\belief)$. The proof is by induction.

\paragraph{Base case ($n=0$)}
We have that $\mathcal{V}^\sched_0(\belief) = 0 = \valuefunc[0](\belief)$.

\paragraph{Induction hypothesis} For fixed $n$ we have
$\mathcal{V}^\sched_n(\belief) \le \valuefunc[n](\belief)$.

\paragraph{Induction step ($n \to n+1$)}
There are four different cases.
\begin{itemize}
    \item If $\sched(\belief) = \action \in \actions$, then
    \begin{align*}
    \mathcal{V}^\sched_{n+1}(\belief) &= \sum_{\belief' \in \clippingstates} \belieftransitions(\belief, \action, \belief') \cdot (\beliefrewards(\belief, \action, \belief') \cdot \mathcal{V}^\sched_{n}(\belief')) \\
    & \le  \sum_{\belief' \in \clippingstates} \belieftransitions(\belief, \action, \belief') \cdot (\beliefrewards(\belief, \action, \belief') \cdot \valuefunc[n](\belief')) = \valuefunc[n+1](\belief).
    \end{align*}
    \item If $\sched(\belief) = \goalaction$ (\ie $\belief \in \goalbels$), then $$\mathcal{V}^\sched_{n+1}(\belief) = 0 = \valuefunc[n+1](\belief).$$
    \item If $\sched(\belief) = \cutaction$, then
    $$\mathcal{V}^\sched_{n+1}(\belief) = \min(\underapprox(\belief), \valuefunc[n+1](\belief)) \le \valuefunc[n+1](\belief).$$
    \item If $\sched(\belief) = \clippingaction$, then consider the path $\belief_m\,\clippingaction\,\belief_{m-1}\,\clippingaction\,\dots\,\clippingaction\,\belief_0$ in $\clippingmdp{\pomdp}$ with $\belief = \belief_m$, $\belief_k \in \clippingstates \setminus \set{\bcut}$ for all $k \ge 0$, $\sched(\belief_k) = \clippingaction$ for all $k > 0$, and $\sched(\belief_0) \neq \clippingaction$. 
    Such a path exists because $\clippingmdp{\pomdp}$ is finite and we do not allow $\clippingaction$-cycles. Furthermore, the path is unique since the action $\clippingaction$ only considers a single successor belief (apart from $\bcut$).
    We show $\mathcal{V}^\sched_{n+1}(\belief_k) \le \valuefunc[n+1](\belief_k)$ using induction over $k$.
    
    Considering $k=0$, we have $\sched(\belief_0) \in \actions \uplus \set{\goalaction, \cutaction}$ and thus the other cases above already yield $\mathcal{V}^\sched_{n+1}(\belief_0) \le \valuefunc[n+1](\belief_0)$.
    
    Now assume $\mathcal{V}^\sched_{n+1}(\belief_k) \le \valuefunc[n+1](\belief_k)$ for a fixed $k$. $\belief_k$ is an adequate clipping candidate for $\belief_{k+1}$ with clipping values $\clippingvalue[] \colonequals \clippingvalue[\belief_{k+1}\to\belief_k]$ and $\stateclippingvalue[](\state) \colonequals \stateclippingvalue[\belief_{k+1}\to\belief_k](\state)$ for $\state \in \supp{\belief_{k+1}}$. We get
        \begin{align*}
            \mathcal{V}^\sched_{n+1}(\belief_{k+1}) 
            &= (1-\clippingvalue[]) \cdot \mathcal{V}^\sched_{n+1}(\belief_k) + \clippingvalue[] \cdot \sum_{\state \in \supp{\belief_{k+1}}} \nicefrac{\stateclippingvalue[](\state)}{\clippingvalue[]} \cdot \stateinfimum_{n+1}(\state)\\
            &\le  (1-\clippingvalue[]) \cdot \valuefunc[n+1](\belief_k) + \clippingvalue[] \cdot \sum_{\state \in \supp{\belief_{k+1}}} \nicefrac{\stateclippingvalue[](\state)}{\clippingvalue[]} \cdot \stateinfimum_{n+1}(\state)\\
            &\le \valuefunc[n+1](\belief_{k+1}).
        \end{align*}
    The last inequality is due to \Cref{lem:clipn}.\qed
\end{itemize}
\end{proof}

\subsection{Proof of \Cref{thm:clippingMDP}}
\corollaryClippingMDP*
\begin{proof}
From \Cref{thm:clippingUnderapprox} it follows that
\begin{align*}
	\exprew{\clippingmdp{\pomdp},\clippingrewards}{\max}{\cutgoals} 
	& = \exprew[\binit]{\clippingmdp{\pomdp},\clippingrewards}{\max}{\cutgoals} \\
	& \leq \exprew[\binit]{\beliefmdp{\pomdp},\beliefrewards}{\max}{\goalbels} \\
	& = \exprew{\beliefmdp{\pomdp},\beliefrewards}{\max}{\goalbels}.
\end{align*}
\qed
\end{proof}

\end{document}